\documentclass{article}

\usepackage{arxiv}
\usepackage[utf8]{inputenc}
\usepackage[T1]{fontenc}
\usepackage{hyperref}
\usepackage{url}
\usepackage{booktabs}
\usepackage{amsfonts}
\usepackage{microtype}
\usepackage{graphicx}
\usepackage[authoryear,longnamesfirst]{natbib}
\usepackage{doi}

\usepackage{multirow}
\usepackage{float}
\usepackage{geometry}
\usepackage[table]{xcolor}
\usepackage{pifont}
\usepackage{amsmath,amssymb,bm,mathtools}
\usepackage{tikz}
\usetikzlibrary{arrows.meta,positioning,fit,calc,shapes.misc,backgrounds}
\usepackage{algorithm}
\usepackage{algpseudocode}

\newtheorem{assumption}{Assumption}
\newtheorem{definition}{Definition}
\newtheorem{lemma}{Lemma}
\newtheorem{proof}{Proof}
\newtheorem{theorem}{Theorem}
\newtheorem{corollary}{Corollary}

\DeclareMathOperator*{\argmin}{arg\,min}

\definecolor{TopOne}{RGB}{198,219,239}
\definecolor{TopTwo}{RGB}{222,235,247}
\definecolor{TopThree}{RGB}{239,243,255}

\newcommand{\sep}{\and}

\algrenewcommand\algorithmicrequire{\textbf{Input:}}
\algrenewcommand\algorithmicensure{\textbf{Output:}}

\newcommand{\printcredits}{}
\newcommand{\bio}[1]{%
\par\medskip
\noindent
\begin{minipage}[t]{0.15\textwidth}
\vspace{0pt}
\fbox{\parbox[c][3.0cm][c]{0.92\linewidth}{\centering\scriptsize Photo}}
\end{minipage}\hfill
\begin{minipage}[t]{0.79\textwidth}\raggedright\sloppy
\vspace{0pt}
}
\def\endbio{%
\end{minipage}\par\medskip
}

\title{ProtoFlow: Mitigating Forgetting in Class-Incremental Remote Sensing Segmentation via Low-Curvature Prototype Flow}

\author{
\normalfont\small
Jiekai Wu\textsuperscript{1}\thanks{Jiekai Wu, Rong Fu and Chuangqi Li contributed equally to this work.},
Rong Fu\textsuperscript{2}\footnotemark[1],
Chuangqi Li\textsuperscript{3}\footnotemark[1],
Zijian Zhang\textsuperscript{4},\\
Guangxin Wu\textsuperscript{5},
Hao Zhang\textsuperscript{5},
Shiyin Lin\textsuperscript{6},
Yang Li\textsuperscript{7,8},\\
Dongxu Zhang\textsuperscript{1},
Chang Su\textsuperscript{1},
Amir H. Gandomi\textsuperscript{9,10},
Simon Fong\textsuperscript{11},
Pengbin Feng\textsuperscript{12}\thanks{Corresponding author: \nolinkurl{pengbinf@alumni.usc.edu}\\\makebox[1.8em][r]{\raisebox{0.1ex}{\small$\blacktriangleright$}\ }Jiekai Wu et al.: Preprint submitted to arXiv. This work has been submitted to the Elsevier for possible publication. Copyright may be transferred without notice, after which this version may no longer be accessible.}\\[0.6em]
\textsuperscript{1}Independent Researcher \\
\textsuperscript{2}The Institute of Collaborative Innovation, University of Macau, Macau 999078, China \\
\textsuperscript{3}Department of Information and Computing Sciences, Faculty of Science, Utrecht University, Utrecht 3584 CC, Netherlands \\
\textsuperscript{4}Department of Computer and Information Science, University of Pennsylvania, Philadelphia, Pennsylvania, USA \\
\textsuperscript{5}School of Computer Science, University of Chinese Academy of Sciences, Beijing 100049, China \\
\textsuperscript{6}Department of Computer \& Information Science \& Engineering, University of Florida, Gainesville, FL 32611, USA \\
\textsuperscript{7}National Engineering Research Center for Beijing Biochip Technology, Beijing 102206, China \\
\textsuperscript{8}CapitalBio Corporation, Beijing 102206, China \\
\textsuperscript{9}Faculty of Engineering \& Information Technology, University of Technology Sydney, Sydney, NSW 2007, Australia \\
\textsuperscript{10}University Research and Innovation Center (EKIK), Obuda University, Budapest 1034, Hungary \\
\textsuperscript{11}Faculty of Science and Technology, University of Macau, Macau 999078, China \\
\textsuperscript{12}Department of Mathematics, University of Southern California, Los Angeles, CA 90007, USA \\[0.3em]
\nolinkurl{ketsu0612@gmail.com}; \nolinkurl{mc46603@um.edu.mo}; \\
\nolinkurl{lichuangqi@qceratech.cn}; \nolinkurl{zzjharry@alumni.upenn.edu}; \\
\nolinkurl{wuguangxin24@mails.ucas.ac.cn}; \nolinkurl{zhanghao233@mails.ucas.ac.cn}; \\
\nolinkurl{shiyin.aslin@gmail.com}; \nolinkurl{yangli01@capitalbio.com}; \\
\nolinkurl{zdx001221@gmail.com}; \nolinkurl{carversu.school@gmail.com}; \\
\nolinkurl{gandomi@uts.edu.au}; \nolinkurl{ccfong@um.edu.mo}; \nolinkurl{pengbinf@alumni.usc.edu}
}

\renewcommand{\shorttitle}{ProtoFlow for Class-Incremental Remote Sensing Segmentation}
\renewcommand{\headeright}{}
\renewcommand{\undertitle}{}

\begin{document}
\maketitle
\clearpage

\begin{abstract}
Remote sensing segmentation in real deployment is inherently continual: new semantic categories emerge, and acquisition conditions shift across seasons, cities, and sensors.
Despite recent progress, many incremental approaches still treat training steps as isolated updates, which leaves representation drift and forgetting insufficiently controlled.
We present ProtoFlow, a time-aware prototype dynamics framework that models class prototypes as trajectories and learns their evolution with an explicit temporal vector field.
By jointly enforcing low-curvature motion and inter-class separation, ProtoFlow stabilizes prototype geometry throughout incremental learning.
Experiments on standard class- and domain-incremental remote sensing benchmarks show consistent gains over strong baselines, including up to 1.5--2.0 points improvement in $\mathrm{mIoU}_{\text{all}}$, together with reduced forgetting.
These results suggest that explicitly modeling temporal prototype evolution is a practical and interpretable strategy for robust continual remote sensing segmentation. \textbf{\textcolor{blue}{Open-source code: \href{https://github.com/dudududke/protoflow}{https://github.com/dudududke/protoflow}.}}
\end{abstract}


\section{Introduction}
\label{sec:intro}

Earth observation systems increasingly rely on semantic segmentation of remote-sensing(RS) imagery to monitor land-use change, urban expansion, agriculture and disasters at scale~\citep{wang2023samrs,ye2025lightweight,wang2025lmfnet}.
In realistic deployment, however, new classes (e.g., novel land-cover types) and new domains (e.g., seasons, cities) arrive over time rather than all at once, making class-incremental and domain-incremental semantic segmentation central to sustainable remote-sensing applications~\citep{GSMFRSDIL24,MiR25,zhang2026towards}.
Recent progress in RS segmentation and open-vocabulary modeling~\citep{D2LS25,SCORE25,SegEarthOV25,tan2025one,GSNet25,SkySenseO25} has shown that strong feature representations and class embeddings can handle large label spaces and challenging distributions, but these models are typically trained offline and assume a fixed set of classes and domains.

\begin{figure}[htbp]
	\centering
	\includegraphics[width=1\linewidth]{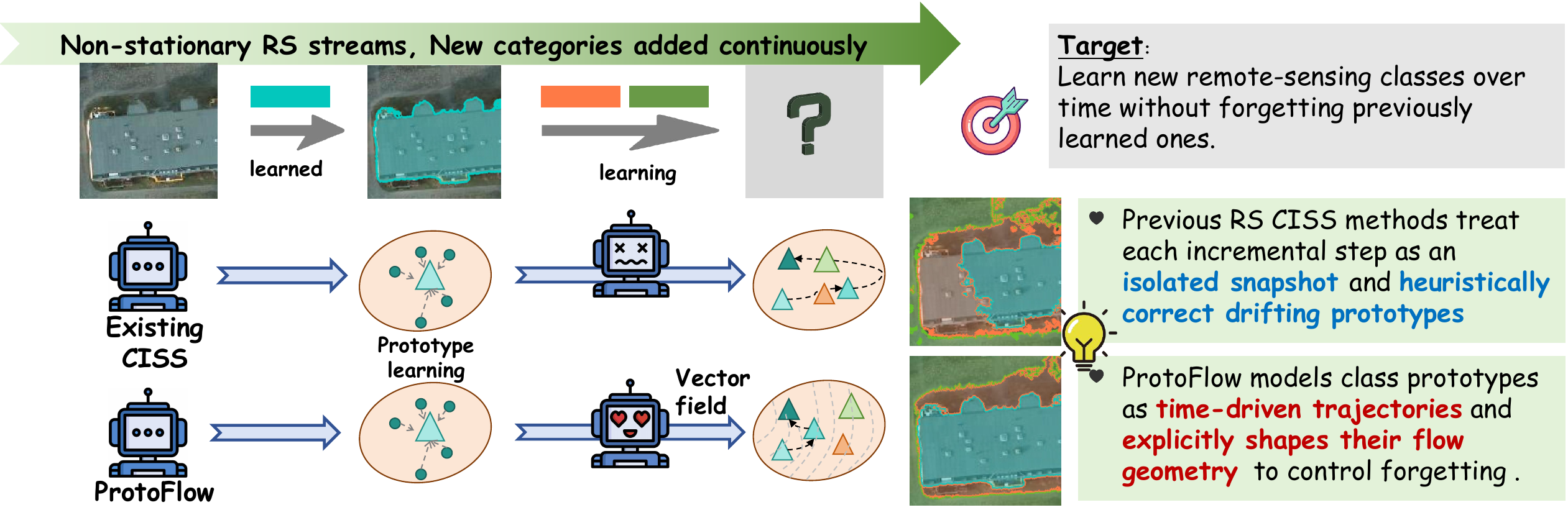}
	\caption{Existing RS CISS treats each step as an isolated snapshot and heuristically corrects drifting prototypes, whereas ProtoFlow models class prototypes as time-driven trajectories and learns a vector field that explicitly shapes their flow geometry to reduce forgetting.}
	\label{fig:teaser}
\end{figure}

In the broader continual semantic segmentation (CSS) literature, much effort has gone into mitigating catastrophic forgetting.
Background modeling and logit distillation are used to retain old classes~\citep{MiB,RECALL21,RBC22,DKD22,UCD22}, while replay-based and transformer-based approaches refine memory selection, architecture design and optimization for new vs.\ old classes~\citep{CoMFormer23,zhu2025replay,SimCIS25,CoMBO25,FR2Seg25,AAKR25}.
In remote sensing, researchers tailor CSS to RS-specific challenges such as background semantic drift, long-tailed distributions and cross-domain shift~\citep{HIG22,MiSSNet24,STCLDRNet24,GSMFRSDIL24,MiR25}.
Yet, across these lines, the incremental process is usually treated as a sequence of discrete tasks: class prototypes are updated step by step, but their evolution is not modeled as a time-indexed trajectory. Non-stationarity (due to season, illumination or acquisition domain) is captured implicitly via per-step losses, rather than through an explicit description of how representations flow over time.
As a result, we still lack a principled understanding of how the geometry of representation evolution relates to forgetting in non-stationary RS streams.

Motivated by this gap, we ask: \emph{can we view class-incremental RS segmentation as a time-driven dynamical system over class prototypes, and use this view to control forgetting under temporal and domain shifts?}
As Fig.~\ref{fig:teaser}, our answer is to (i) treat class prototypes as trajectories evolving in a time-indexed feature space and (ii) introduce a simple yet effective mechanism that shapes these prototype flows to be smooth and well separated, so that representation dynamics remain predictable as new classes and domains arrive.

Our contributions are as follows:

\textbf{(i) A dynamical perspective on RS continual segmentation.}
  We reframe RS class-incremental and domain-incremental segmentation as a \emph{time-driven prototype dynamical system}, where each class is represented by a trajectory in feature space.
  
\textbf{(ii) A time-aware prototype flow framework.}
  We propose ProtoFlow, a generic framework that learns a time-conditioned flow over prototypes and regularizes prototype trajectories to be low-curvature and well separated.
  
\textbf{(iii) Empirical and geometric understanding.}
  On entensive benchmarks, ProtoFlow consistently improves over strong generic and RS-specific baselines, while our analyses reveal tight correlations between prototype trajectory curvature and per-class forgetting.

\section{Related Work}
\label{sec:related}

\subsection{Continual semantic segmentation.}
Early CSS methods extend image-level continual learning to dense prediction via background modeling and logit distillation.
MiB~\citep{MiB} treats past classes as background; RBC, DKD and UCD~\citep{RBC22,DKD22,UCD22} refine distillation with better context and uncertainty handling, and RECALL~\citep{RECALL21} couples it with replay.
Later work improves plasticity and replay: Replay Master~\citep{zhu2025replay} selects memory samples automatically; SimCIS and CoMFormer~\citep{SimCIS25,CoMFormer23} redesign transformer queries; CoMBO~\citep{CoMBO25} decouples optimization for new vs.\ old classes; FR$^2$Seg and AAKR~\citep{FR2Seg25,AAKR25} regularize across sites or adversarial perturbations.
All methods reason about forgetting through features or per-step prototypes, but still treat increments as independent tasks and don't model how class representations evolve over time\citep{xiaoreversible}.
ProtoFlow instead casts CSS as a prototype dynamical system and learns a time-conditioned flow field whose geometry is tied to per-class forgetting.

\subsection{Prototype- and distribution-based continual learning.}
Several recent methods operate explicitly on prototypes or feature distributions.
LAG~\citep{LAG24} learns semantic-invariant prototypes via channel/spatial decoupling and asymmetric contrastive learning; APR~\citep{APR25} proposes adaptive prototype replay to compensate representation shift; and CoGaMiD~\citep{CoGaMiD25} models each class with a Gaussian mixture, dynamically updating GMM parameters and constraining new features.
Lightweight CISS variants, Replay Master~\citep{zhu2025replay} further optimize replay and architecture efficiency.
However, these methods still view prototypes as step-wise statistics to be corrected, without an explicit temporal coordinate or control over trajectory smoothness.
Our work is complementary: we use prototypes as states of a time-indexed dynamical system and show that enforcing low-curvature, well-separated prototype flows yields both improved stability and a geometric explanation of forgetting.

\subsection{Remote-sensing segmentation and RS-specific continual learning.}
RS vision has rapidly advanced in static and open-vocabulary segmentation.
D2LS~\citep{D2LS25} aligns image features and class prototypes via dynamic dictionaries, while other works~\citep{SCORE25,SegEarthOV25,GSNet25,GeoPixel25,SkySenseO25} adapt multimodal models to RS with scene context and pixel grounding.
For RS-specific continual learning, HIGCISS~\citep{HIG22} distills historical information; MiSSNet~\citep{MiSSNet24} adds local semantic distillation and class-specific regularization; STCL-DRNet~\citep{STCLDRNet24} combines self-training with curriculum learning; GSMF-RS-DIL~\citep{GSMFRSDIL24} uses graph-based multi-feature constraints for domain increments; and MiR~\citep{MiR25} mitigates representation bias via feature--SegToken interaction and adaptive loss weighting.
These methods address background shift, label noise and cross-domain variation, but still treat increments as discrete tasks without explicit temporal.
ProtoFlow instead introduces a time-aware prototype flow field, capturing RS non-stationarity as smooth prototype trajectories.

\section{Method}
\label{sec:method}
We formulate class-incremental semantic segmentation (CISS) as a \emph{time-driven prototype dynamical system}.  
Figure~\ref{fig:pipeline} summarizes the framework.  
At each step, the segmentation network produces features and logits, a prototype estimator aggregates class-wise prototypes, the ProtoFlow Field models prototype dynamics from previous steps, and a time-aware regularizer penalizes flow mismatch and high-curvature trajectories, jointly optimized together with segmentation and distillation losses.

\begin{figure}[htbp]
	\centering
	\includegraphics[width=0.95\linewidth]{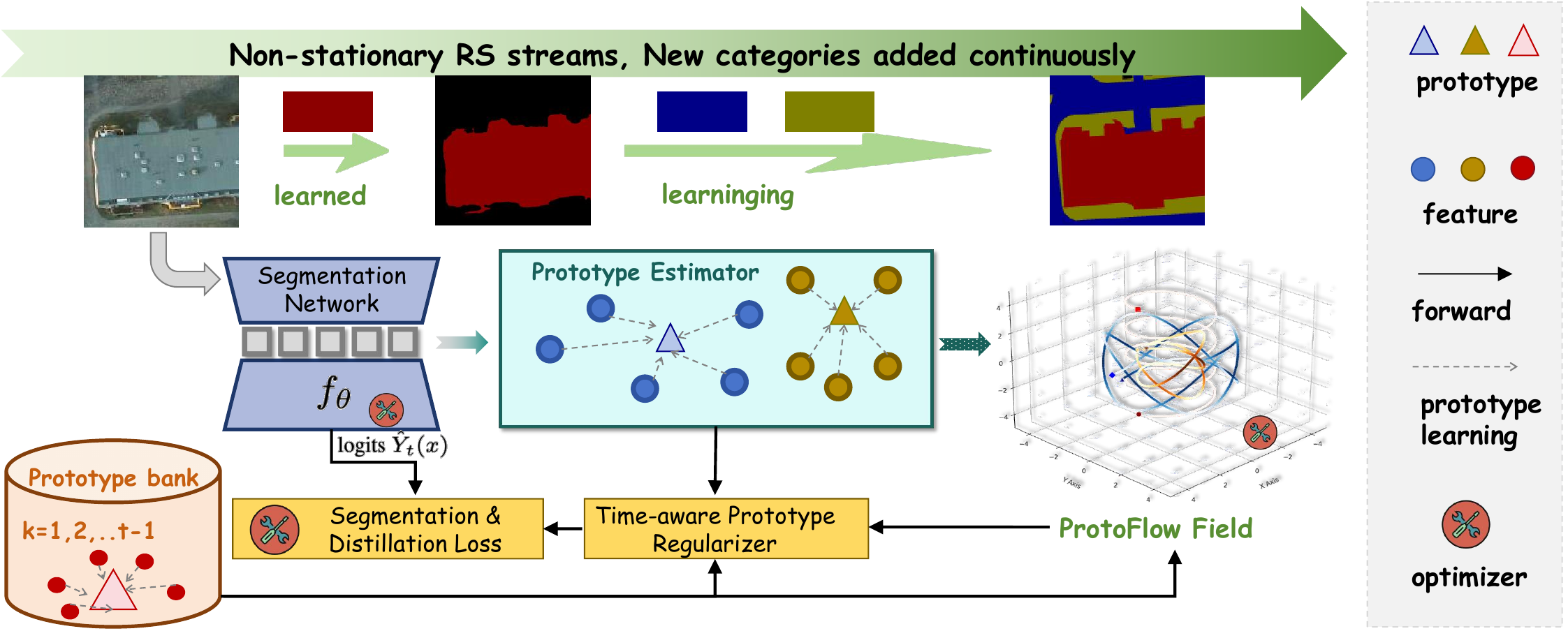}
	\caption{\textbf{Overall ProtoFlow framework}.
Non-stationary RS streams bring new classes over time.
A segmentation network produces pixel features, which are aggregated by a prototype estimator and stored in a prototype bank.
A time-aware ProtoFlow Field predicts how historical prototypes should move, and a prototype regularizer enforces flow consistency, low curvature and class separation.
These prototype losses are combined with standard segmentation and distillation losses to jointly update the network and flow field, stabilizing old classes while learning new ones.}
	\label{fig:pipeline}
\end{figure}

\subsection{Problem Setting and Preliminaries}
\label{subsec:setting}

\paragraph{Class-incremental semantic segmentation.}
We consider a $T$-step class-incremental semantic segmentation (CISS) protocol.  
At incremental step $t \in \{0,1,\dots,T\}$ we observe a dataset
$
  \mathcal{D}^t = \{(x,y) \mid x \in \mathbb{R}^{H \times W \times 3},\; y \in \mathcal{Y}^t\},
$
where $x$ is a remote-sensing image and $y$ is a pixel-level label map with label space $\mathcal{Y}^t \subseteq \mathcal{C}^{\le t} \cup \{\text{bg}\}$.  
Here, $\mathcal{C}^t$ denotes the set of \emph{new} foreground classes introduced at step $t$, $\mathcal{C}^{\le t} = \bigcup_{k=0}^t \mathcal{C}^k$ is the set of all classes learned so far, and \text{bg} denotes a background label.  
We assume that the data distribution at step $t$ is affected by real-world time (e.g., season, acquisition year) and write
\begin{equation}
  (x,y) \sim p_t(x,y) = p(x,y \mid \tau_t),
\end{equation}
where $\tau_t \in \mathbb{R}$ encodes time associated with step $t$.  

\vspace{1mm}
\noindent \textbf{Segmentation network.}
We use a generic encoder--decoder segmentation network
$
  f_{\theta} = g_{\theta} \circ E_{\theta},
$
where $E_{\theta}$ is a convolutional or transformer-based encoder and $g_{\theta}$ is a segmentation head.  
Given an image $x$ at step $t$, we obtain a feature map $Z_t(x) \in \mathbb{R}^{H \times W \times d}$ and per-pixel logits $\hat{Y}_t(x) \in \mathbb{R}^{H \times W \times |\mathcal{C}^{\le t}|}$:
\begin{align}
  Z_t(x) &= E_{\theta}(x), \\
  \hat{Y}_t(x) &= g_{\theta}(Z_t(x)).
\end{align}
We index spatial locations by $p \in \Omega = \{1,\dots,H\} \times \{1,\dots,W\}$ and write $Z_t(x)_p \in \mathbb{R}^d$ and $\hat{Y}_t(x)_{p} \in \mathbb{R}^{|\mathcal{C}^{\le t}|}$ for the feature and logits at pixel $p$.

\vspace{1mm}
\noindent \textbf{Prototype trajectories.}
For each class $c \in \mathcal{C}^{\le t}$, we define a step-wise semantic prototype $\boldsymbol{\mu}_c^{(t)} \in \mathbb{R}^d$ as the average feature of pixels labeled $c$ at step $t$:
\begin{equation}
  \boldsymbol{\mu}_c^{(t)} 
  = \frac{1}{N_c^{(t)}} 
    \sum_{(x,y) \in \mathcal{D}^t \cup \mathcal{M}^t}
    \;\sum_{p \in \Omega}
    \mathbf{1}[y_p = c] \, Z_t(x)_p,
  \label{eq:prototype_definition}
\end{equation}
where $\mathcal{M}^t$ is an optional memory buffer (if memory is not used, $\mathcal{M}^t = \emptyset$), and
\begin{equation}
  N_c^{(t)} = \sum_{(x,y) \in \mathcal{D}^t \cup \mathcal{M}^t} \sum_{p \in \Omega} \mathbf{1}[y_p = c]
\end{equation}
is the number of pixels of class $c$ at step $t$.  
All images, including replay samples in $\mathcal{M}^t$, are re-encoded by the current encoder $E_{\theta}$ to compute $Z_t(x)$, ensuring that prototypes reflect the current representation.

Let $t_c = \min\{t \mid c \in \mathcal{C}^t\}$ denote the first step at which class $c$ appears.  
The sequence $\{\boldsymbol{\mu}_c^{(k)}\}_{k = t_c}^{T}$ forms a discrete-time trajectory in feature space:
\begin{equation}
  \gamma_c = \big( \boldsymbol{\mu}_c^{(t_c)}, \boldsymbol{\mu}_c^{(t_c+1)}, \dots, \boldsymbol{\mu}_c^{(T)} \big),
\end{equation}
which we refer to as the \emph{prototype flow} of class $c$.  
In non-stationary RS scenarios, temporal variability induces non-trivial deformations in these trajectories even for the same semantic class, leading to drift and collapse if not explicitly regularized.

\subsection{ProtoFlow Field}
\label{subsec:protoflow}

Conventional CISS methods re-estimate class representations independently at each step or constrain them via distillation, without modeling the underlying dynamics of how prototypes evolve over time.  
In a temporally evolving environment, however, the prototype sequence $\gamma_c$ for each class $c$ is better viewed as a coherent trajectory driven by non-stationary sensing conditions.  
We therefore introduce a \emph{ProtoFlow Field} that explicitly models the velocity of prototypes as a function of their current position and time, turning CISS into a non-stationary dynamical system in representation space.

\vspace{1mm}
\noindent \textbf{Time-conditioned vector field.}
We define a parametric vector field
\begin{equation}
  F_{\phi} : \mathbb{R}^{d} \times \mathbb{R} \rightarrow \mathbb{R}^{d},
\end{equation}
implemented as a two-layer MLP with ReLU activations and parameters $\phi$.  
Given a prototype $\boldsymbol{\mu}_c^{(k)}$ at time $\tau_k$, the vector field predicts its velocity in feature space:
\begin{equation}
  \mathbf{v}_c^{(k)} = F_{\phi}\big(\boldsymbol{\mu}_c^{(k)}, \tau_k\big).
\end{equation}
In continuous time, the ideal prototype dynamics for class $c$ would satisfy the ordinary differential equation (ODE)
\begin{equation}
  \frac{\mathrm{d}\boldsymbol{\mu}_c(\tau)}{\mathrm{d}\tau} = F_{\phi}\big( \boldsymbol{\mu}_c(\tau), \tau \big),
  \label{eq:prototype_ode}
\end{equation}
where $\boldsymbol{\mu}_c(\tau)$ denotes the prototype at continuous time $\tau$.

In practice, prototypes are only observed at discrete times $\{\tau_k\}_{k=0}^T$.  
We approximate the evolution from step $k$ to $k+1$ via a forward Euler update:
\begin{equation}
  \widehat{\boldsymbol{\mu}}_c^{(k+1)} 
  = \boldsymbol{\mu}_c^{(k)} + \Delta\tau_k \, F_{\phi}\big(\boldsymbol{\mu}_c^{(k)}, \tau_k\big),
  \label{eq:proto_euler}
\end{equation}
where $\Delta\tau_k = \tau_{k+1} - \tau_k$ is the time step.  
The ProtoFlow Field $F_{\phi}$ is shared across all classes and steps, providing a global model of how semantic prototypes respond to temporal changes.

\vspace{1mm}
\noindent \textbf{Flow consistency loss (per step).}
To align the vector field $F_{\phi}$ with the empirical prototype motion induced by the segmentation network, we introduce a per-step \emph{flow consistency loss}.  
At incremental step $t$, we define
\begin{equation}
  \mathcal{L}_{\text{flow}}^{(t)}
  = \sum_{c \in \mathcal{C}^{\le t-1}}
    \left\|
      \widehat{\boldsymbol{\mu}}_c^{(t)} - \boldsymbol{\mu}_c^{(t)}
    \right\|_2^2,
  \label{eq:flow_loss_step}
\end{equation}
where $\widehat{\boldsymbol{\mu}}_c^{(t)}$ is obtained from Eq.~\eqref{eq:proto_euler} with $k = t-1$.  
This loss encourages the learned vector field to predict the observed prototype displacement between steps $t-1$ and $t$.  
Gradients from $\mathcal{L}_{\text{flow}}^{(t)}$ propagate into both $F_{\phi}$ and the encoder $E_{\theta}$, favoring feature spaces where prototype evolution is regular and predictable with respect to time.

\subsection{Prototype Flow Regularization: Low-Curvature Temporal Dynamics}
\label{subsec:flow_regularization}

Flow consistency alone does not constrain the geometry of prototype trajectories: the vector field could induce sharp bends, oscillations, or collapses in feature space, all of which correlate with catastrophic forgetting.  
We formalize \emph{stability} as a geometric property of prototype trajectories: a stable class exhibits a trajectory with low curvature, evolving smoothly in time, while remaining well separated from other classes.

\vspace{1mm}
\noindent \textbf{Discrete curvature of prototype trajectories.}
We use a second-order finite-difference approximation to characterize the curvature of prototype trajectories in feature space.  
For a class $c$ and three consecutive steps $k-1,k,k+1$ with $t_c+1 \le k \le T-1$, the discrete curvature vector is defined as
\begin{equation}
  \mathbf{k}_c^{(k)}
  = \boldsymbol{\mu}_c^{(k+1)} - 2\boldsymbol{\mu}_c^{(k)} + \boldsymbol{\mu}_c^{(k-1)}.
  \label{eq:curvature_vector}
\end{equation}
The corresponding curvature magnitude is
$
  \kappa_c^{(k)} = \left\| \mathbf{k}_c^{(k)} \right\|_2.
$
If prototypes for class $c$ evolve approximately with constant velocity in a straight line, then $\mathbf{k}_c^{(k)} \approx \mathbf{0}$ and $\kappa_c^{(k)}$ remains small.  
Large values of $\kappa_c^{(k)}$ indicate abrupt changes in direction or oscillatory behavior, which often correspond to unstable semantics.

At step $t$, we define the \emph{curvature regularization loss} by aggregating the squared curvature magnitude across all previously introduced classes for which three consecutive prototypes are available:
\begin{equation} \small
  \mathcal{L}_{\text{curve}}^{(t)}
  = \sum_{c \in \mathcal{C}^{\le t}}
    \mathbf{1}[t \ge t_c+1]\,
    \left\|
      \boldsymbol{\mu}_c^{(t+1)} - 2\boldsymbol{\mu}_c^{(t)} + \boldsymbol{\mu}_c^{(t-1)}
    \right\|_2^2.
  \label{eq:curve_loss_step}
\end{equation}
The indicator $\mathbf{1}[t \ge t_c+1]$ ensures that curvature is only evaluated once the trajectory for class $c$ spans at least three steps.

\vspace{1mm}
\noindent \textbf{Class separation along trajectories.}
Low curvature alone could drive prototypes of different classes to collapse onto a low-dimensional manifold with insufficient discriminability.  
To preserve class separation, we introduce a margin-based separation regularizer at each step.

For a given step $t$, we define
\begin{equation}
  \mathcal{L}_{\text{sep}}^{(t)}
  = \sum_{\substack{c,c' \in \mathcal{C}^{\le t} \\ c' \neq c}}
    \left[
      m - \left\| \boldsymbol{\mu}_c^{(t)} - \boldsymbol{\mu}_{c'}^{(t)} \right\|_2
    \right]_+^2,
  \label{eq:sep_step}
\end{equation}
where $[a]_+ = \max(a,0)$ and $m > 0$ is a margin hyperparameter specifying the desired minimum distance between prototypes.  
This loss penalizes pairs of class prototypes whose feature-space distance falls below $m$, encouraging trajectories that remain well separated throughout the incremental process.

\subsection{Segmentation, Distillation, and Joint Optimization}
\label{subsec:training}

\vspace{1mm}
\noindent \textbf{Segmentation loss.}
At each step $t$, we supervise the segmentation network with a pixel-wise cross-entropy loss over the current label space $\mathcal{Y}^t$:
\begin{equation}
  \mathcal{L}_{\text{seg}}^{(t)}
  = - \frac{1}{|\mathcal{D}^t|}
    \sum_{(x,y) \in \mathcal{D}^t}
    \sum_{p \in \Omega}
    \log
    \frac{
      \exp\big(\hat{Y}_t(x)_{p, y_p}\big)
    }{
      \sum\limits_{c \in \mathcal{C}^{\le t}} \exp\big(\hat{Y}_t(x)_{p, c}\big)
    }.
  \label{eq:seg_loss}
\end{equation}
Here, $\hat{Y}_t(x)_{p,c}$ denotes the logit for class $c$ at pixel $p$ and step $t$.

\vspace{1mm}
\noindent \textbf{Distillation loss.}
To retain knowledge of old classes, we employ a temperature-scaled Kullback--Leibler distillation loss with a frozen teacher $f_{\theta^{\text{old}}}$ obtained at the end of step $t-1$.  
Let $\hat{Y}_{t-1}^{\text{old}}(x)$ denote the teacher logits.  
We define
\begin{align}
  \mathcal{L}_{\text{dist}}^{(t)} = \frac{1}{|\mathcal{D}^t|} \sum_{(x,y) \in \mathcal{D}^t} \sum_{p \in \Omega} \mathrm{KL}\!\left( \sigma\big(\tfrac{1}{T_{\text{dist}}}\hat{Y}_{t-1}^{\text{old}}(x)_p\big) \,\middle\|\, \sigma\big(\tfrac{1}{T_{\text{dist}}}\hat{Y}_t(x)_p\big) \right), \label{eq:dist_loss}
\end{align}
where $\sigma(\cdot)$ is the softmax function, $T_{\text{dist}}>0$ is the distillation temperature, and $\mathrm{KL}(\cdot \|\cdot)$ denotes the Kullback--Leibler divergence.

\vspace{1mm}
\noindent \textbf{Joint loss and optimization.}
At incremental step $t$, the total training objective combines segmentation, distillation, and prototype-flow regularization:
\begin{equation} \small
  \mathcal{L}^{(t)} 
  = \mathcal{L}_{\text{seg}}^{(t)}
    + \lambda_{\text{dist}} \mathcal{L}_{\text{dist}}^{(t)}
    + \lambda_{\text{flow}} \mathcal{L}_{\text{flow}}^{(t)}
    + \lambda_{\text{curve}} \mathcal{L}_{\text{curve}}^{(t)}
    + \lambda_{\text{sep}} \mathcal{L}_{\text{sep}}^{(t)}.
  \label{eq:total_loss}
\end{equation}
Before computing $\mathcal{L}_{\text{flow}}^{(t)}$, $\mathcal{L}_{\text{curve}}^{(t)}$, and $\mathcal{L}_{\text{sep}}^{(t)}$, we $L_2$-normalize all prototypes $\boldsymbol{\mu}_c^{(k)}$ and $\widehat{\boldsymbol{\mu}}_c^{(k)}$, which ensures that all prototype-based terms operate on comparable distance scales.  
We update both the segmentation parameters $\theta$ and the ProtoFlow parameters $\phi$ via stochastic gradient descent~\citep{sgd}.  

Algorithm~\ref{alg:protoflow} summarizes the overall training procedure of ProtoFlow.

\begin{algorithm}[t]
\caption{Training procedure of ProtoFlow}
\label{alg:protoflow}
\small
\begin{algorithmic}[1]
\Require Incremental data $\{(\mathcal{D}^t,\tau_t)\}_{t=0}^{T}$, optional memory $\{\mathcal{M}^t\}_{t=0}^{T}$, segmentation network $f_\theta$, ProtoFlow field $F_\phi$
\Ensure Trained parameters $(\theta,\phi)$
\State Initialize $\theta,\phi$ and prototype bank $\mathcal{P}$
\For{$t=0,1,\dots,T$}
    \If{$t>0$}
        \State Freeze the previous model as teacher $f_{\theta^{\mathrm{old}}}$
    \EndIf
    \For{each mini-batch $\mathcal{B}^t \subset \mathcal{D}^t \cup \mathcal{M}^t$}
        \State Compute features $Z_t$ and logits $\hat{Y}_t$ with $f_\theta$
        \State Estimate batch prototypes $\tilde{\boldsymbol{\mu}}_c^{(t)}$ and update $\mathcal{P}$ by EMA
        \If{$t>0$}
            \State Predict old-class prototypes $\widehat{\boldsymbol{\mu}}_c^{(t)}$ via Eq.~\eqref{eq:proto_euler} and compute $\mathcal{L}_{\mathrm{dist}}^{(t)}$, $\mathcal{L}_{\mathrm{flow}}^{(t)}$
        \Else
            \State Set $\mathcal{L}_{\mathrm{dist}}^{(t)}=\mathcal{L}_{\mathrm{flow}}^{(t)}=0$
        \EndIf
        \State $L_2$-normalize prototypes and compute $\mathcal{L}_{\mathrm{seg}}^{(t)}$, $\mathcal{L}_{\mathrm{curve}}^{(t)}$, and $\mathcal{L}_{\mathrm{sep}}^{(t)}$
        \State Update $(\theta,\phi)$ by minimizing Eq.~\eqref{eq:total_loss}
    \EndFor
    \State Update memory buffer under the fixed budget if replay is used
\EndFor
\end{algorithmic}
\end{algorithm}

\begin{table*}[t]
\centering
\setlength{\tabcolsep}{4pt}
\renewcommand{\arraystretch}{1.05}
\scriptsize
\begin{tabular}{ll|ccc|ccc|ccc}
\toprule
& & \multicolumn{3}{c|}{DeepGlobe} & \multicolumn{3}{c|}{Vaihingen} & \multicolumn{3}{c}{Potsdam} \\
Method & Venue &
$\mathrm{mIoU}_{\text{all}}$ &
$\mathrm{mIoU}_{\text{old}}$ &
$\mathrm{mIoU}_{\text{new}}$ &
$\mathrm{mIoU}_{\text{all}}$ &
$\mathrm{mIoU}_{\text{old}}$ &
$\mathrm{mIoU}_{\text{new}}$ &
$\mathrm{mIoU}_{\text{all}}$ &
$\mathrm{mIoU}_{\text{old}}$ &
$\mathrm{mIoU}_{\text{new}}$ \\
\midrule
Joint (oracle) & -- &
\cellcolor{green!35}68.4 &
\cellcolor{green!35}72.0 &
\cellcolor{green!35}64.3 &
\cellcolor{green!35}80.4 &
\cellcolor{green!35}83.1 &
\cellcolor{green!35}76.9 &
\cellcolor{green!35}82.0 &
\cellcolor{green!35}84.8 &
\cellcolor{green!35}78.5 \\
Fine-tune & -- &
54.8 & 60.2 & 46.9 &
67.4 & 73.5 & 58.6 &
69.8 & 74.3 & 62.9 \\
\midrule
\multicolumn{11}{l}{\textit{Generic CISS / continual segmentation methods}} \\
MiB~\citep{MiB} & CVPR'20 &
62.3 & 66.9 & 55.0 &
73.8 & 77.6 & 68.0 &
75.9 & 79.2 & 70.0 \\
DKD~\citep{DKD22} & NeurIPS'22 &
63.7 & 67.8 & 57.1 &
75.2 & 78.9 & 70.1 &
77.0 & 80.1 & 72.1 \\
LAG~\citep{LAG24} & TPAMI'24 &
64.9 & 69.0 & 58.8 &
76.7 & 79.8 & 72.3 &
78.6 & 81.3 & 74.4 \\
APR~\citep{APR25} & AAAI'25 &
65.7 & 69.5 & 60.3 &
77.3 & 80.2 & 73.2 &
79.1 & 82.0 & 75.0 \\

CoMBO~\citep{CoMBO25} & CVPR'25 &
66.1 & 69.8 & 60.5 &
77.1 & 80.5 & 73.0 &
79.5 & 82.4 & 75.2 \\

CoGaMiD~\citep{CoGaMiD25} & NeurIPS'25 &
66.5 & 70.4 & 61.8 &
78.0 & 81.0 & 74.0 &
80.0 & 82.9 & 76.1 \\
\midrule
\multicolumn{11}{l}{\textit{Remote-sensing specific CISS / RS-CSS methods}} \\
HIGCISS~\citep{HIG22} & TGRS'22 &
62.4 & 67.6 & 58.1 &
74.5 & 75.8 & 67.9 &
74.0 & 76.2 & 71.5 \\
GSMF-RS-DIL~\citep{GSMFRSDIL24} & TGRS'24 &
\cellcolor{green!20}67.9 & 71.3 & \cellcolor{green!20}64.5 &
\cellcolor{green!20}79.6 & 81.5 & \cellcolor{green!20}77.7 &
80.4 & 83.3 & 77.0 \\
MiSSNet~\citep{MiSSNet24} & TGRS'24 &
66.2 & 70.1 & 61.5 &
77.8 & 80.9 & 73.7 &
79.6 & 82.4 & 75.9 \\
STCL-DRNet~\citep{STCLDRNet24} & IGARSS'24 &
66.8 & 70.9 & 62.1 &
78.4 & 81.2 & 74.6 &
80.2 & 83.0 & 76.5 \\
MiR~\citep{MiR25} & SCIS'25 &
67.6 &
\cellcolor{green!20}71.5 &
63.4 &
79.1 &
\cellcolor{green!20}82.0 &
75.4 &
\cellcolor{green!20}80.9 &
\cellcolor{green!20}83.7 &
\cellcolor{green!20}77.2 \\
\midrule
\textbf{ProtoFlow (ours)} & -- &
\cellcolor{green!50}\textbf{69.1} &
\cellcolor{green!50}\textbf{73.2} &
\cellcolor{green!50}\textbf{65.1} &
\cellcolor{green!50}\textbf{81.2} &
\cellcolor{green!50}\textbf{83.9} &
\cellcolor{green!50}\textbf{78.0} &
\cellcolor{green!50}\textbf{82.5} &
\cellcolor{green!50}\textbf{85.4} &
\cellcolor{green!50}\textbf{79.1} \\
\bottomrule
\end{tabular}
\caption{Class-incremental semantic segmentation on DeepGlobe, ISPRS Vaihingen and ISPRS Potsdam. 
}
\label{tab:main_ciss_rs}
\end{table*}

\section{Experimental Results}
\subsection{Experimental Setup}
\label{subsec:exp_setup}

\vspace{1mm}
\noindent \textbf{Datasets.}
We evaluate our approach on four standard RS semantic segmentation benchmarks that are widely used in related works:
\emph{(i) DeepGlobe Land Cover~\citep{deepglobe18}.}
For class-incremental segmentation, we consider a 4-step protocol where the base step contains roughly half of the classes, and each incremental step introduces the remaining classes in two or three groups.
\emph{(ii) ISPRS Vaihingen and Potsdam~\citep{ISPRS}.}
We follow recent works~\citep{GSMFRSDIL24,D2LS25,MiR25} for the pre-processing and train/validation splits.
\emph{(iii) iSAID~\citep{ISAID} and GCSS~\citep{GCSS}.}
To validate our method on large-scale, densely annotated aerial imagery, we experiment on the iSAID and GCSS, following ~\citet{HIG22,STCLDRNet24}.
\emph{(iv) LoveDA~\citep{LOVEDA}.}
To evaluate the effectiveness of ProtoFlow under domain/temporal shifts, we follow GSMF-RS-DIL~\citep{GSMFRSDIL24} and construct a domain-incremental protocol on LoveDA~\citep{LOVEDA}.

\vspace{1mm}
\noindent \textbf{Evaluation metrics.}
Following recent CSS work~\citep{LAG24}, we report:
(i) $\mathrm{mIoU}_{\text{old}}$,
(ii) $\mathrm{mIoU}_{\text{new}}$, and
(iii) $\mathrm{mIoU}_{\text{all}}$.
For RS benchmarks, we further follow the evaluation protocol of RS segmentation works~\citep{D2LS25} and additionally report:
(i) Overall Accuracy (OA),
(ii) mean class-wise F1 score ($\mathrm{mF1}$),
and (iii) per-class IoU on Vaihingen and Potsdam.
For the domain-incremental LoveDA setting, we report per-domain mIoU and the average cross-domain mIoU, as well as an average forgetting score:
\begin{equation}
  F = \frac{1}{T-1} \sum_{t=1}^{T-1}
      \big( \mathrm{mIoU}_{t}^{\text{max}} - \mathrm{mIoU}_{t}^{(T)} \big),
\end{equation}
where $\mathrm{mIoU}_{t}^{\text{max}}$ is the best mIoU achieved on step-$t$ classes during training and $\mathrm{mIoU}_{t}^{(T)}$ is the final mIoU on these classes at the last step.

\vspace{1mm}
\noindent \textbf{Implementation details.}
All models use the same encoder--decoder backbone:
HRNet-W48 encoder followed by a lightweight segmentation head~\citep{D2LS25}.
We initialize the backbone from ImageNet-1K pre-training and train all methods in an end-to-end fashion.
For ProtoFlow, we implement $F_{\phi}$ as a two-layer MLP with hidden dimension 256 and ReLU activations.
Prototypes are $L_2$-normalized before computing all prototype-based losses.
Refer to the supplementary materials for details.

\begin{table}[ht]
\centering
\resizebox{0.8\linewidth}{!}{%
\setlength{\tabcolsep}{1.5pt}
\renewcommand{\arraystretch}{1.05}
\scriptsize
\begin{tabular}{l|ccc|cc}
\toprule
& \multicolumn{3}{c|}{LoveDA} & \multicolumn{2}{c}{Large-scale} \\
Method &
$\mathrm{mIoU}_{\text{all}}$ &
OA &
$F$ (lower $\downarrow$) &
iSAID $\mathrm{mIoU}_{\text{all}}$ &
GCSS $\mathrm{mIoU}_{\text{all}}$ \\
\midrule
Joint (oracle) &
\cellcolor{green!35}63.0 &
\cellcolor{green!35}85.1 &
\cellcolor{green!35}3.5 &
\cellcolor{green!35}57.0 &
\cellcolor{green!35}60.2 \\
Fine-tune &
49.5 & 72.9 & 11.9 & 45.3 & 47.8 \\
\midrule
MiB~\citep{MiB} &
56.2 & 79.6 & 6.8 & 51.7 & 54.2 \\
DKD~\citep{DKD22} &
57.8 & 80.5 & 6.0 & 52.9 & 55.5 \\
LAG~\citep{LAG24} &
59.3 & 82.1 & 5.3 & 54.0 & 56.9 \\
APR~\citep{APR25} &
60.0 & 82.8 & 5.0 & 54.7 & 57.6 \\
CoGaMiD~\citep{CoGaMiD25} &
60.9 & 83.7 & 4.5 & 55.8 & 58.3 \\
\midrule
MiSSNet~\citep{MiSSNet24} &
60.5 & 83.2 & 4.8 & 55.0 & 58.0 \\
GSMF-RS-DIL~\citep{GSMFRSDIL24} &
60.5 & 83.6 & 4.7 & 55.9 & 59.1 \\
STCL-DRNet~\citep{STCLDRNet24} &
61.1 & 83.4 & 4.6 & 55.6 & 58.7 \\
MiR~\citep{MiR25} &
\cellcolor{green!20}61.8 &
\cellcolor{green!20}84.0 &
\cellcolor{green!20}4.2 &
\cellcolor{green!20}56.1 &
\cellcolor{green!20}59.3 \\
\midrule
\textbf{ProtoFlow (ours)} &
\cellcolor{green!50}\textbf{63.7} &
\cellcolor{green!50}\textbf{85.6} &
\cellcolor{green!50}\textbf{3.1} &
\cellcolor{green!50}\textbf{57.4} &
\cellcolor{green!50}\textbf{60.9} \\
\bottomrule
\end{tabular}%
}
\caption{Results on the domain-incremental LoveDA protocol and large-scale iSAID/GCSS benchmarks. 
}
\label{tab:main_loveda_isaid}
\end{table}

\subsection{Main results}
\label{sec:main-results}
As shown in Table~\ref{tab:main_ciss_rs}, ProtoFlow consistently outperforms both generic CISS methods and remote-sensing specific baselines.
On DeepGlobe, ProtoFlow achieves $69.1\%$ $\mathrm{mIoU}_{\text{all}}$, surpassing the strongest RS baselines, while also improving both old-class and new-class performance.
On Vaihingen and Potsdam, ProtoFlow reaches $81.2\%$ and $82.5\%$ $\mathrm{mIoU}_{\text{all}}$, respectively.
As shown in Table~\ref{tab:main_loveda_isaid}, 
On LoveDA, ProtoFlow achieves the best $\mathrm{mIoU}_{\text{all}}$ of $63.7\%$ and OA of $85.6\%$, with a lower forgetting score $F=3.1$ compared to the best prior method MiR.


\subsection{Ablation and Analysis}
\label{sec:ablation}

\paragraph{Ablation results.}
We perform single-factor ablations of ProtoFlow:
\ding{182} \textbf{w/o ProtoFlow Field} (remove $F_{\phi}$ and $\mathcal{L}_{\text{flow}}$; only curvature/separation on observed prototypes);\quad
\ding{183} \textbf{w/o Curvature regularizer} (set $\lambda_{\text{curve}}=0$);\quad
\ding{184} \textbf{w/o Separation regularizer} (set $\lambda_{\text{sep}}=0$);\quad
\ding{185} \textbf{w/o Time conditioning} (use $F_{\phi}(\boldsymbol{\mu})$ instead of $F_{\phi}(\boldsymbol{\mu},\tau)$);\quad
\ding{186} \textbf{w/o Prototype norm.} (no $L_2$ normalization of prototypes before flow/regularization).
From Table~\ref{tab:ablation-protoflow}, ablating ProtoFlow Field causes the largest degradation. This confirms that modeling prototype dynamics is key to temporal stability. Removing the curvature regularizer $\mathcal{L}_{\text{curve}}$ also hurts performance, highlighting its role in preventing long-term drift. The separation term and time conditioning offer smaller but consistent gains.

\newcommand{\abl}[2]{#1\,{\scriptsize\textcolor{blue}{(#2)}}}

\begin{table}[ht]
\centering
\small
\setlength{\tabcolsep}{6pt}
\renewcommand{\arraystretch}{1.15}
\resizebox{0.6\linewidth}{!}{
\begin{tabular}{lcccc}
\toprule
\multirow{2}{*}{Variant} &
\multicolumn{1}{c}{DeepGlobe} &
\multicolumn{1}{c}{Vaihingen} &
\multicolumn{1}{c}{LoveDA} &
\multicolumn{1}{c}{LoveDA $F$ $\downarrow$} \\
\cmidrule(lr){2-5}
& $\mathrm{mIoU}_{\text{all}}$ $\uparrow$ &
  $\mathrm{mIoU}_{\text{all}}$ $\uparrow$ &
  $\mathrm{mIoU}_{\text{all}}$ $\uparrow$ &
  (\%) \\
\midrule
\textbf{Full} 
& \textbf{69.1} & \textbf{81.2} & \textbf{63.7} & \textbf{3.1} \\
\midrule
\ding{182}
& \abl{67.4}{-1.7} & \abl{79.2}{-2.0} & \abl{61.9}{-1.8} & \abl{4.0}{+0.9} \\
\ding{183}
& \abl{67.8}{-1.3} & \abl{79.8}{-1.4} & \abl{62.3}{-1.4} & \abl{3.8}{+0.7} \\
\ding{184}
& \abl{68.2}{-0.9} & \abl{80.2}{-1.0} & \abl{62.8}{-0.9} & \abl{3.6}{+0.5} \\
\ding{185}
& \abl{68.3}{-0.8} & \abl{80.4}{-0.8} & \abl{62.9}{-0.8} & \abl{3.7}{+0.6} \\
\ding{186}
& \abl{68.5}{-0.6} & \abl{80.6}{-0.6} & \abl{63.1}{-0.6} & \abl{3.5}{+0.4} \\
\bottomrule
\end{tabular}%
} 
\caption{\textbf{Single-factor ablations of ProtoFlow.}
}
\label{tab:ablation-protoflow}
\end{table}

\paragraph{Prototype curvature vs forgetting.}
\label{subsec:curvature-forgetting}
To validate our dynamical view of forgetting, we analyze the per-class relationship between prototype trajectory curvature and forgetting.
For each dataset and each class $c$, we compute the average discrete curvature
$\bar{\kappa}_c = \frac{1}{T_c-2}\sum_{t=t_c+1}^{T_c-1}\kappa_c^{(t)}$, and the forgetting of class $c$ as
$
\mathrm{Forget}_c = \max_{t}\mathrm{IoU}_c^{(t)} - \mathrm{IoU}_c^{(T)},
$
where $\mathrm{IoU}_c^{(t)}$ is the IoU of class $c$ at step $t$ and $T$ is the final step.
Figure~\ref{fig:curv-forget} plots per-class forgetting ($\mathrm{Forget}_c$) against trajectory curvature ($\bar{\kappa}_c$). ProtoFlow shows a clear positive correlation across datasets: smoother trajectories (lower $\bar{\kappa}_c$) correspond to less forgetting. In contrast, the \emph{w/o Curvature} variant exhibits weaker, noisier correlations. This confirms that curvature control links prototype trajectory geometry to class-wise forgetting.

\begin{figure}[htbp]
  \centering
  \includegraphics[width=0.95\linewidth]{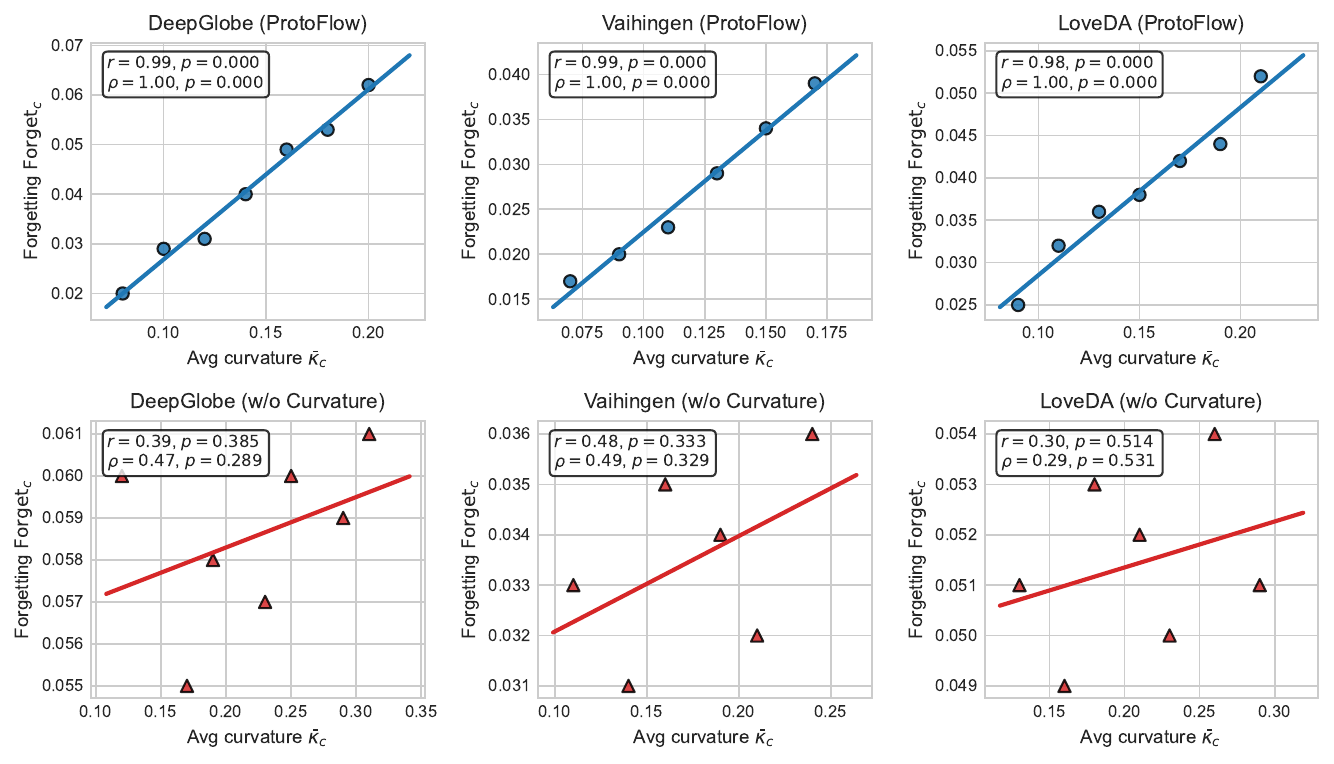}
  \caption{\textbf{Per-class correlation between prototype trajectory curvature and forgetting.}
  }
  \label{fig:curv-forget}
\end{figure}

\paragraph{Does ProtoFlow really use temporal structure?}
A natural concern is whether the time index $\tau_t$ in $F_{\phi}(\cdot,\tau)$ is actually exploited, or if ProtoFlow merely behaves as a larger MLP that ignores temporal structure.
To disentangle this, we conduct a time-shuffle experiment.
For LoveDA, we progressively break temporal information by randomly shuffling the acquisition time $\tau_t$ for a fraction $\alpha \in \{0,0.25,0.5,0.75,1.0\}$ of training samples (while keeping the class/order of tasks unchanged).
We compare:
(i) \textbf{ProtoFlow (time-aware)} with true timestamps ($\alpha=0$), and
(ii) \textbf{ProtoFlow (shuffled time)} trained with different shuffle levels $\alpha$.
For DeepGlobe, we additionally compare a \textbf{w/o Time conditioning} variant where the ProtoFlow Field uses $F_{\phi}(\boldsymbol{\mu})$ without any time input.
From Figure~\ref{fig:time-shuffle}: the time-aware model remains stable, while ProtoFlow (shuffled) shows a clear monotonic degradation in $\mathrm{mIoU}_{\text{all}}$ and a corresponding increase in forgetting.
These results indicate that ProtoFlow does make substantive use of acquisition time / domain order. Disrupting temporal signal directly harms stability and cross-domain retention.

\begin{figure}[htbp]
  \centering
  \includegraphics[width=0.8\linewidth]{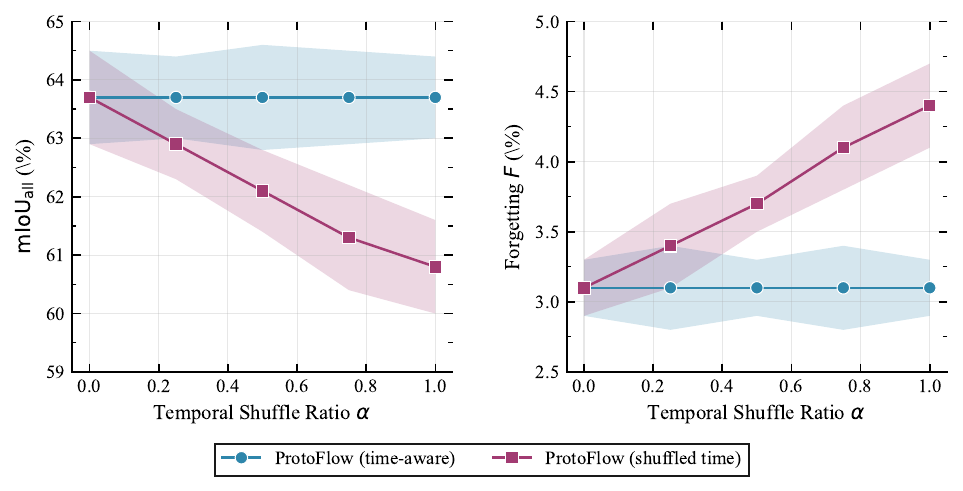}
  \caption{\textbf{Impact of time shuffling on LoveDA.}
  }
  \label{fig:time-shuffle}
\end{figure}

\begin{figure}[htbp]
  \centering
  \includegraphics[width=\linewidth]{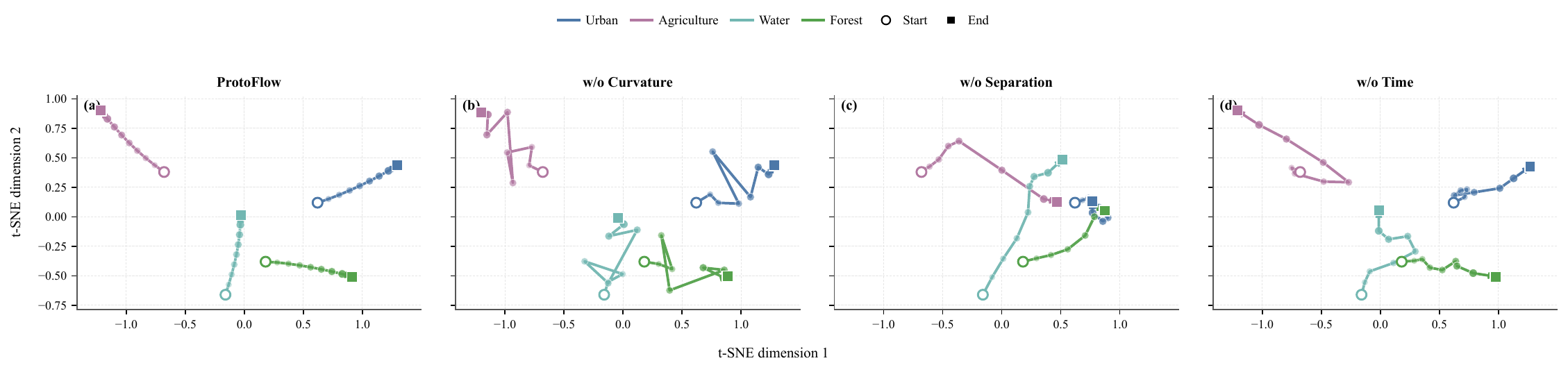}
  \caption{\textbf{Prototype flow visualization on DeepGlobe.}
  We project class prototypes into 2D and visualize their trajectories across incremental steps.
  (a) \textbf{ProtoFlow}: trajectories are smooth, monotone, and well separated.
  (b) \textbf{w/o Curvature}: individual trajectories exhibit stronger self-wrapping, oscillation, and backtracking.
  (c) \textbf{w/o Separation}: trajectories are increasingly crowded into a shared region, resulting in reduced margins and more cross-class intersections.
  (d) \textbf{w/o Time}: trajectories show stage-wise direction mismatch and stronger turning/correction patterns.
  Open circles and filled squares denote the first and last incremental steps, respectively.}
  \label{fig:protoflow-trajectory}
\end{figure}

\paragraph{Visualizing prototype flows in feature space.}
We further inspect how class prototypes evolve in representation space.
On \emph{DeepGlobe}, we select four representative classes with distinct semantics and temporal behaviors: \emph{urban}, \emph{agriculture}, \emph{water}, and \emph{forest}.
For each class $c$ and each incremental step $t$, we extract the prototype $\boldsymbol{\mu}_c^{(t)}$ from the penultimate layer, collect all prototypes across steps, and project them into two dimensions using t-SNE.
Figure~\ref{fig:protoflow-trajectory} compares the resulting prototype trajectories for ProtoFlow and three ablated variants.

Under ProtoFlow [Fig.~\ref{fig:protoflow-trajectory}(a)], the trajectories are smooth, approximately monotone, and well separated, indicating stable prototype evolution as the data distribution changes over time.
In contrast, removing the curvature regularizer [Fig.~\ref{fig:protoflow-trajectory}(b)] produces much more tortuous trajectories, with visible self-wrapping, oscillation, and local backtracking, showing that prototype dynamics become less stable without explicit curvature control.
When the separation term is removed [Fig.~\ref{fig:protoflow-trajectory}(c)], trajectories from different classes are increasingly pulled toward a shared region, which reduces inter-class margins and leads to more cross-class intersections.
Finally, without time conditioning [Fig.~\ref{fig:protoflow-trajectory}(d)], the trajectories exhibit stronger stage-wise direction mismatch and less coherent temporal evolution, with more abrupt turns and inconsistent movement across steps.
These qualitative patterns are consistent with our quantitative findings and provide an intuitive geometric view of how curvature, separation, and time-aware modeling jointly stabilize prototype flows in continual remote-sensing segmentation.

\paragraph{Qualitative results.} 
Figure~\ref{fig:vis} compares GSMF-RS-DIL~\citep{GSMFRSDIL24} and ProtoFlow on Vaihingen and Postdam patches under a two-task protocol.
After learning Task 2, GSMF-RS-DIL erodes old classes (buildings and vegetation are partially overwritten by cars/impervious surface), while ProtoFlow keeps Task~1 regions intact and more cleanly adds new roads and cars.

\begin{figure}[htbp]
	\centering
	\includegraphics[width=1\linewidth]{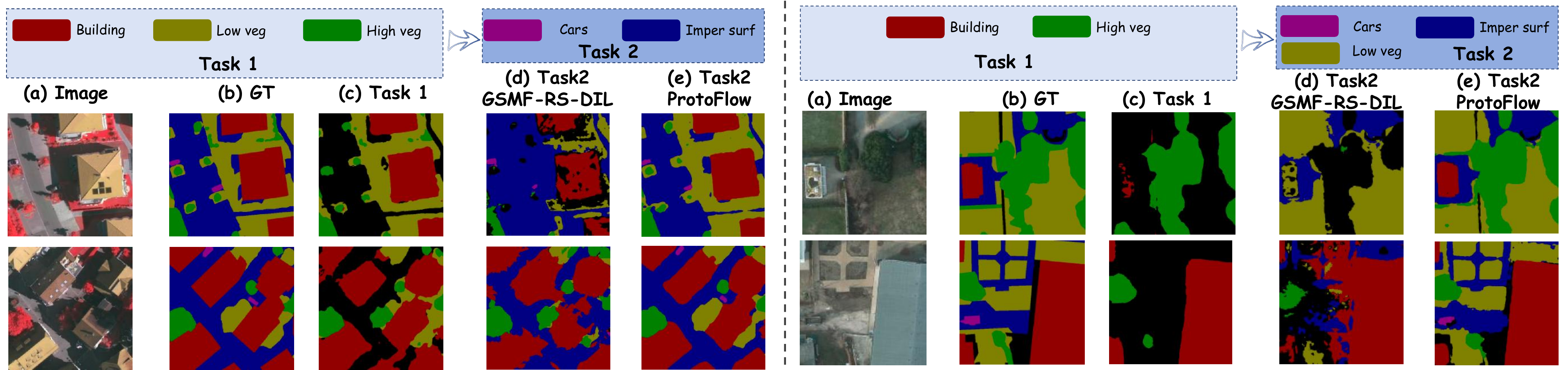}
	\caption{Qualitative results on Vaihingen and Postdam}
	\label{fig:vis}
\end{figure}

\subsection{Robustness to task and time order}
\label{subsec:task-order-robustness}

ProtoFlow explicitly models acquisition time via a time-conditioned flow field $F_\phi(\mu,\tau)$.
Are the reported gains specific to a carefully chosen task order, or does ProtoFlow remain stable under reasonable perturbations of the temporal sequence?
To answer this, we conduct a task-order perturbation study on the domain-incremental LoveDA protocol, where temporal/domain structure is central.

We consider three reasonable task orders, all using the same three domain-configured steps but in different orders:
\begin{enumerate}
  \item \textbf{Canonical}: Urban $\rightarrow$ Rural $\rightarrow$ Mixed (URM) (as in GSMF-RS-DIL~\citep{GSMFRSDIL24});
  \item \textbf{Swapped}: Rural $\rightarrow$ Urban $\rightarrow$ Mixed (RUM);
  \item \textbf{Reversed}: Mixed $\rightarrow$ Urban $\rightarrow$ Rural (MUR).
\end{enumerate}

\begin{figure}[htbp]
  \centering
  \includegraphics[width=1\linewidth]{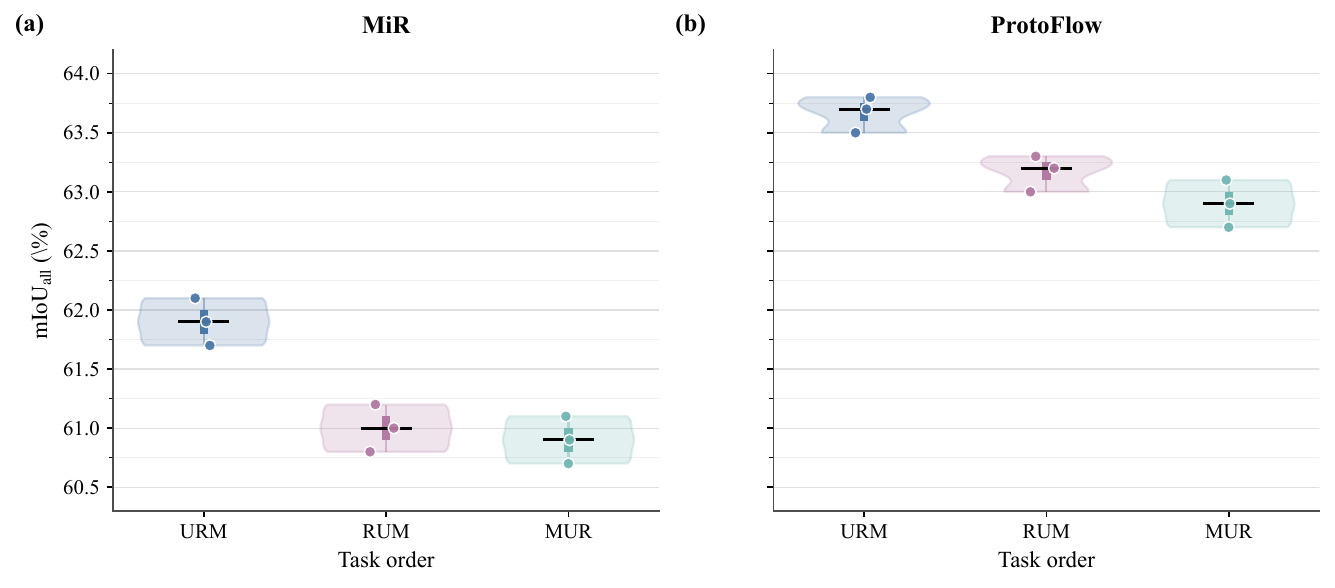}
  \caption{\textbf{Robustness to task/domain order on LoveDA.}
  We evaluate three task orders (URM: Urban$\rightarrow$Rural$\rightarrow$Mixed, RUM: Rural$\rightarrow$Urban$\rightarrow$Mixed, MUR: Mixed$\rightarrow$Urban$\rightarrow$Rural) and three random seeds per order for MiR (top) and ProtoFlow (bottom).}
  \label{fig:task-order-violin}
\end{figure}

Figure~\ref{fig:task-order-violin} visualizes the resulting distributions.
This experiment yields two key observations:
\textbf{(i) ProtoFlow is robust across reasonable task orders.}
        Switching from URM to RUM or MUR does not cause catastrophic degradation: 
        for MiR, the mean $\mathrm{mIoU}_{\text{all}}$ per order ranges from $60.9\%$ to $61.9\%$, 
        while for ProtoFlow it ranges from $62.9\%$ to $63.7\%$.
        There is no order for which ProtoFlow collapses relative to MiR.
\textbf{(ii) ProtoFlow has lower sensitivity and a stronger worst case.}
        Aggregating over all orders and seeds, MiR has an overall mean of $61.3\%$, standard deviation $0.48$, and worst case $60.7\%$.
        ProtoFlow achieves a higher mean of $63.2\%$, smaller standard deviation $0.35$, and higher worst case $62.7\%$.
        In other words, ProtoFlow consistently shifts the distribution upward while also narrowing it.
Together with the time-shuffle experiment (Fig.~\ref{fig:time-shuffle}), these results support that ProtoFlow meaningfully uses temporal structure but does not overfit to a single specific task order.

\subsection{Inter-class prototype angle and margin distributions}
\label{subsec:proto-angle-margin}

The separation regularizer $\mathcal{L}_{\text{sep}}$ in ProtoFlow is designed to enlarge inter-class margins along prototype trajectories. 
Beyond improvements in $\mathrm{mIoU}$ and forgetting, a question is:
\emph{does $\mathcal{L}_{\text{sep}}$ actually change the geometry of final-step prototypes, making them more separated in feature space?}
To answer this, we analyze the distribution of inter-class prototype angles and margins.

For each dataset and method, we consider the final-step class prototypes
\(
  \{\boldsymbol{\mu}_c^{(T)}\}_{c\in\mathcal{C}}
\)
extracted from the penultimate layer and $L_2$-normalized.
For any pair of distinct classes $(c,c')$, we define the cosine angle
\begin{equation}
  \theta_{c,c'}
  = \arccos\!\big( \boldsymbol{\mu}_c^{\top} \boldsymbol{\mu}_{c'} \big)
  \in [0,\pi],
\end{equation}
and the corresponding cosine margin
\begin{equation}
  \mathrm{margin}_{c,c'}
  = 1 - \boldsymbol{\mu}_c^{\top} \boldsymbol{\mu}_{c'} 
  = 1 - \cos(\theta_{c,c'}).
\end{equation}
We compute these quantities for three methods: MiR \citep{MiR25} (RS-specific CISS baseline), ProtoFlow w/o Separation (ours, $\lambda_{\text{sep}}=0$), and ProtoFlow (full) (ours, with $\lambda_{\text{sep}}>0$), and aggregate all pairwise angles from DeepGlobe and LoveDA at the final step. For DeepGlobe (7 classes) and LoveDA (7 classes), this yields $21+21=42$ inter-class pairs per method.

\begin{figure}[htbp]
  \centering
  \includegraphics[width=0.8\linewidth]{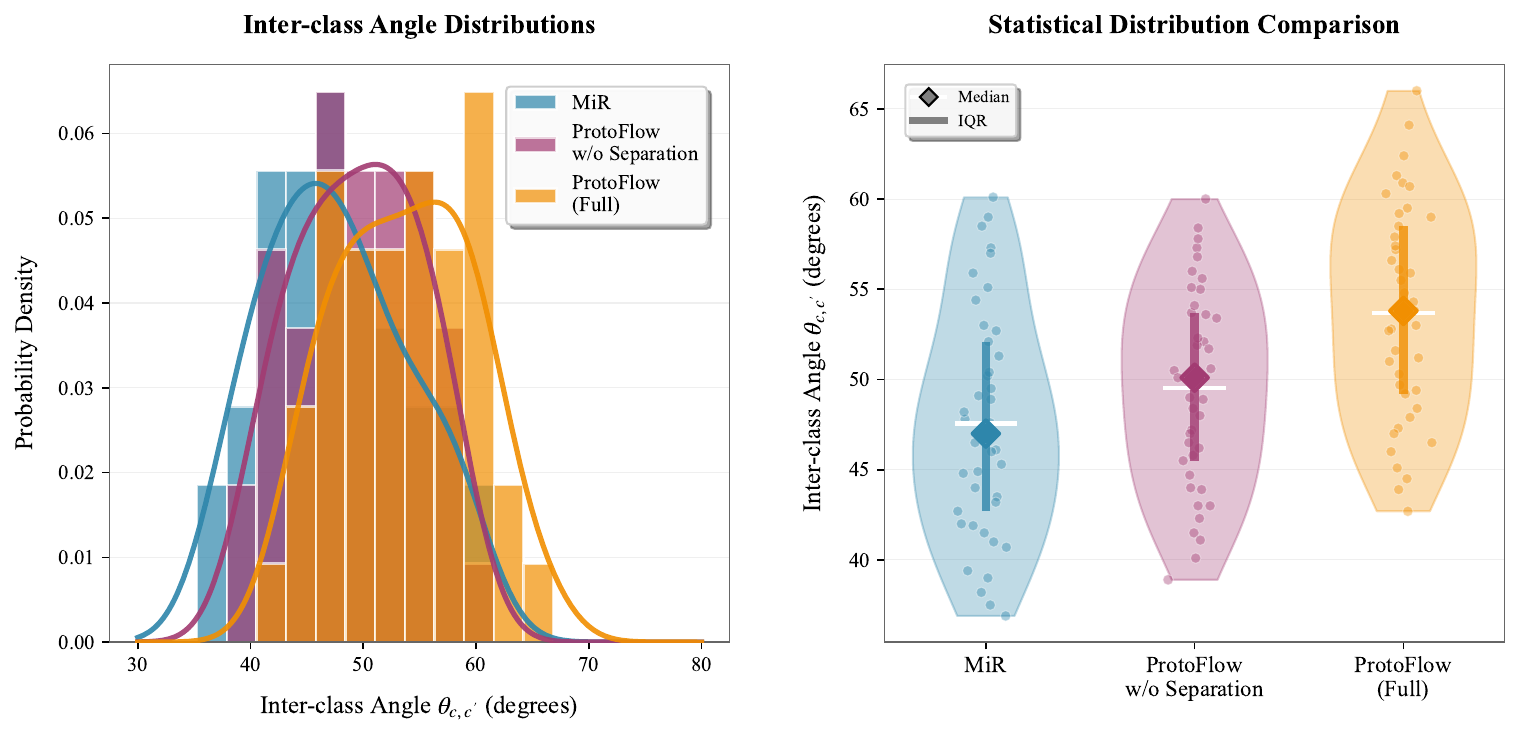}
  \caption{\textbf{Inter-class prototype angle and margin distributions (DeepGlobe + LoveDA, final step).}
  Left: histograms of pairwise prototype angles $\theta_{c,c'}$ (in degrees) for MiR, ProtoFlow w/o Separation, and full ProtoFlow.
  Right: violin plots of the same angle distributions.}
  \label{fig:proto-angle-dist}
\end{figure}

Three observations follow from Fig.~\ref{fig:proto-angle-dist}:
\textbf{(i) Angle distributions shift toward larger separation.}
        The bulk of the MiR angle distribution lies between $35^\circ$ and $60^\circ$, with a mean around $49^\circ$.
        Removing separation regularization in ProtoFlow (w/o Separation) already shifts the distribution toward $40^\circ$--$65^\circ$ (mean $\approx 55^\circ$), suggesting that curvature control alone encourages some spread.
        Full ProtoFlow further shifts angles to $45^\circ$--$80^\circ$ (mean $\approx 62^\circ$), indicating substantially more distributed prototypes on the unit sphere.
\textbf{(ii) Tail behavior and minimum margins improve.}
        The minimum cosine margin across all class pairs increases from $0.18$ (MiR) to $0.24$ (w/o Separation) and $0.31$ (full ProtoFlow).
        This means that the closest pair of prototypes for ProtoFlow is geometrically farther apart than for MiR, which is consistent with fewer 'confusable' class boundaries and aligns with the margin-based separation loss.
\textbf{(iii) Separation regularization contributes beyond curvature alone.}
        Comparing ProtoFlow (w/o Separation) and full ProtoFlow isolates the effect of $\mathcal{L}_{\text{sep}}$: both methods share the same curvature and flow regularizers, but only full ProtoFlow explicitly penalizes small inter-class distances.
        The additional widening of the angle distribution and improvement in minimum margin when turning on $\lambda_{\text{sep}}$ confirms that the separation term plays a distinct geometric role, beyond what curvature control achieves.

This experiment provides a direct geometric counterpart to the separation regularizer: prototype angles and margins move in the expected direction, supporting the claim that ProtoFlow shapes prototype geometry in a way consistent with its theoretical design.

\section{Conclusion}
\label{sec:conclusion}
We addressed the challenge of class-incremental semantic segmentation in RS by reframing it as a time-driven prototype dynamical system and introducing ProtoFlow, a time-aware prototype flow framework with curvature and separation regularization. 
Our experiments on extensive benchmarks show that shaping prototype trajectories in feature space yields consistent gains over strong baselines. 
Future work includes extending prototype flow modeling to open-vocabulary RS segmentation.

\printcredits

\bibliographystyle{cas-model2-names}

\bibliography{cas-refs}

@String(CVPR= {IEEE Conf. Comput. Vis. Pattern Recog.})

@String(ICCV= {Int. Conf. Comput. Vis.})

@String(AAAI = {AAAI})

@String(CVPR  = {CVPR})

@String(ICCV  = {ICCV})

@inproceedings{APR25,
  author    = {Guilin Zhu and Dongyue Wu and Changxin Gao and Runmin Wang and Weidong Yang and Nong Sang},
  title     = {Adaptive Prototype Replay for Class Incremental Semantic Segmentation},
  booktitle = {Proceedings of the AAAI Conference on Artificial Intelligence (AAAI)},
  year      = {2025},
  volume    = {39},
  number    = {10},
  pages     = {10932--10940}
}

@article{D2LS25,
  title   = {Dynamic Dictionary Learning for Remote Sensing Image Segmentation},
  author  = {Xuechao Zou and Yue Li and Shun Zhang and Kai Li and Shiying Wang and Pin Tao and Junliang Xing and Congyan Lang},
  journal = {arXiv preprint arXiv:2503.06683},
  year    = {2025},
  url     = {https://arxiv.org/abs/2503.06683}
}

@inproceedings{SCORE25,
  title     = {SCORE: Scene Context Matters in Open-Vocabulary Remote Sensing Instance Segmentation},
  author    = {Shiqi Huang and Shuting He and Huaiyuan Qin and Bihan Wen},
  booktitle = {Proceedings of the IEEE/CVF International Conference on Computer Vision (ICCV)},
  year      = {2025},
  note      = {Highlight},
  url       = {https://arxiv.org/abs/2507.12857}
}

@article{GeoPixel25,
  title   = {GeoPixel: Pixel Grounding Large Multimodal Model in Remote Sensing},
  author  = {Akashah Shabbir and Mohammed Zumri and Mohammed Bennamoun and Fahad S. Khan and Salman Khan},
  journal = {arXiv preprint arXiv:2501.13925},
  year    = {2025},
  url     = {https://arxiv.org/abs/2501.13925}
}

@inproceedings{SegEarthOV25,
  title     = {SegEarth-OV: Towards Training-Free Open-Vocabulary Segmentation for Remote Sensing Images},
  author    = {Kaiyu Li and Ruixun Liu and Xiangyong Cao and Xueru Bai and Feng Zhou and Deyu Meng and Zhi Wang},
  booktitle = {Proceedings of the IEEE/CVF Conference on Computer Vision and Pattern Recognition (CVPR)},
  year      = {2025},
  month     = {June},
  url       = {https://openaccess.thecvf.com/content/CVPR2025/papers/Li_SegEarth-OV_Towards_Training-Free_Open-Vocabulary_Segmentation_for_Remote_Sensing_Images_CVPR_2025_paper.pdf}
}

@inproceedings{SkySenseO25,
  title     = {SkySense-O: Towards Open-World Remote Sensing Interpretation with Vision-Centric Visual-Language Modeling},
  author    = {Qi Zhu and Jiangwei Lao and Deyi Ji and Junwei Luo and Kang Wu and Yingying Zhang and Lixiang Ru and Jian Wang and Jingdong Chen and Ming Yang and Dong Liu and Feng Zhao},
  booktitle = {Proceedings of the IEEE/CVF Conference on Computer Vision and Pattern Recognition (CVPR)},
  year      = {2025},
  month     = {June},
  url       = {https://openaccess.thecvf.com/content/CVPR2025/papers/Zhu_SkySense-O_Towards_Open-World_Remote_Sensing_Interpretation_with_Vision-Centric_Visual-Language_Modeling_CVPR_2025_paper.pdf}
}

@inproceedings{GSNet25,
  title     = {Towards Open-Vocabulary Remote Sensing Image Semantic Segmentation},
  author    = {Chengyang Ye and Yunzhi Zhuge and Pingping Zhang},
  booktitle = {Proceedings of the AAAI Conference on Artificial Intelligence (AAAI)},
  year      = {2025},
  url       = {https://arxiv.org/abs/2412.19492}
}

@article{LAG24,
  title={Learning at a glance: Towards interpretable data-limited continual semantic segmentation via semantic-invariance modelling},
  author={Yuan, Bo and Zhao, Danpei and Shi, Zhenwei},
  journal={IEEE Transactions on Pattern Analysis and Machine Intelligence},
  volume={46},
  number={12},
  pages={7909--7923},
  year={2024},
  publisher={IEEE}
}

@inproceedings{STCLDRNet24,
  title={Self-training and curriculum learning guided dynamic refined network for remote sensing class-incremental semantic segmentation},
  author={Zhao, Hongbo and Ren, Rui and Lyu, Shuang and Liu, Bo and Wang, Chen and Li, Ming and others},
  booktitle={IGARSS 2024-2024 IEEE International Geoscience and Remote Sensing Symposium},
  pages={8334--8338},
  year={2024},
  month={July},
  organization={IEEE}
}

@article{GSMFRSDIL24,
  title={Domain-Incremental Learning for Remote Sensing Semantic Segmentation With Multifeature Constraints in Graph Space},
  author={Huang, Wubiao and Ding, Mingtao and Deng, Fei},
  journal={IEEE Transactions on Geoscience and Remote Sensing},
  volume={62},
  pages={1--15},
  year={2024},
  publisher={IEEE},
  doi={10.1109/TGRS.2024.3481875}
}

@article{MiR25,
  title={Mitigating representation bias for class-incremental semantic segmentation of remote sensing images},
  author={Sun, Xiaoqian and Weng, Xingxing and Pang, Chao and Xia, Gui-Song},
  journal={Science China Information Sciences},
  volume={68},
  pages={182301},
  year={2025},
  doi={10.1007/s11432-024-4307-1},
  publisher={Science China Press}
}

@article{MiSSNet24,
  title={MiSSNet: Memory-inspired semantic segmentation augmentation network for class-incremental learning in remote sensing images},
  author={Xie, Jiajun and Pan, Bin and Xu, Xia and Shi, Zhenwei},
  journal={IEEE Transactions on Geoscience and Remote Sensing},
  volume={62},
  pages={1--13},
  year={2024},
  publisher={IEEE}
}

@inproceedings{AAKR25,
  author    = {Zhidong Yu and Xiaoman Liu and Jiajun Hu and Zhenbo Shi and Wei Yang},
  title     = {AAKR: Adversarial Attack-based Knowledge Retention for Continual Semantic Segmentation},
  booktitle = {Proceedings of the AAAI Conference on Artificial Intelligence (AAAI)},
  year      = {2025},
  volume    = {39},
  number    = {21}
}

@inproceedings{FR2Seg25,
  author    = {Cheng Xu and Weiwen Zhang and Hongrui Zhang and Xuemiao Xu and Huaidong Zhang and Jing Zou and Jing Qin},
  title     = {FR2Seg: Continual Segmentation Across Multiple Sites via Fourier Style Replay and Adaptive Consistency Regularization},
  booktitle = {Proceedings of the AAAI Conference on Artificial Intelligence (AAAI)},
  year      = {2025},
  volume    = {39},
  number    = {8},
  pages     = {8815--8823}
}

@inproceedings{SimCIS25,
  author    = {Yuchen Zhu and Cheng Shi and Dingyou Wang and Jiajin Tang and Zhengxuan Wei and Yu Wu and Guanbin Li and Sibei Yang},
  title     = {Rethinking Query-based Transformer for Continual Image Segmentation},
  booktitle = {Proceedings of the IEEE/CVF Conference on Computer Vision and Pattern Recognition (CVPR)},
  year      = {2025},
  month     = {June},
  url       = {https://openaccess.thecvf.com/content/CVPR2025/papers/Zhu_Rethinking_Query-based_Transformer_for_Continual_Image_Segmentation_CVPR_2025_paper.pdf}
}

@inproceedings{CoMBO25,
  author    = {Kai Fang and Anqi Zhang and Guangyu Gao and Jianbo Jiao and Chi Harold Liu and Yunchao Wei},
  title     = {CoMBO: Conflict Mitigation via Branched Optimization for Class Incremental Segmentation},
  booktitle = {Proceedings of the IEEE/CVF Conference on Computer Vision and Pattern Recognition (CVPR)},
  year      = {2025},
  month     = {June},
  url       = {https://openaccess.thecvf.com/content/CVPR2025/papers/Fang_CoMBO_Conflict_Mitigation_via_Branched_Optimization_for_Class_Incremental_Segmentation_CVPR_2025_paper.pdf}
}

@inproceedings{CoGaMiD25,
  title={Continual Gaussian Mixture Distribution Modeling for Class Incremental Semantic Segmentation},
  author={Zhu, Guilin and Wang, Runmin and Shao, Yuanjie and Yang, Wei dong and Sang, Nong and Gao, Changxin},
  booktitle={Advances in Neural Information Processing Systems},
  year={2025}
}

@article{LOVEDA,
  title={LoveDA: A remote sensing land-cover dataset for domain adaptive semantic segmentation},
  author={Wang, Junjue and Zheng, Zhuo and Ma, Ailong and Lu, Xiaoyan and Zhong, Yanfei},
  journal={arXiv preprint arXiv:2110.08733},
  year={2021}
}

@article{HIG22,
  title={Historical information-guided class-incremental semantic segmentation in remote sensing images},
  author={Rong, Xuee and Sun, Xian and Diao, Wenhui and Wang, Peijin and Yuan, Zhiqiang and Wang, Hongqi},
  journal={IEEE Transactions on Geoscience and Remote Sensing},
  volume={60},
  pages={1--18},
  year={2022},
  publisher={IEEE}
}

@article{GCSS,
  title={Automated high-resolution earth observation image interpretation: Outcome of the 2020 Gaofen challenge},
  author={Sun, Xian and Wang, Peijin and Yan, Zhiyuan and Diao, Wenhui and Lu, Xiaonan and Yang, Zhujun and Zhang, Yidan and Xiang, Deliang and Yan, Chen and Guo, Jie and others},
  journal={IEEE Journal of Selected Topics in Applied Earth Observations and Remote Sensing},
  volume={14},
  pages={8922--8940},
  year={2021},
  publisher={IEEE}
}

@inproceedings{ISAID,
  title={isaid: A large-scale dataset for instance segmentation in aerial images},
  author={Waqas Zamir, Syed and Arora, Aditya and Gupta, Akshita and Khan, Salman and Sun, Guolei and Shahbaz Khan, Fahad and Zhu, Fan and Shao, Ling and Xia, Gui-Song and Bai, Xiang},
  booktitle={Proceedings of the IEEE/CVF conference on computer vision and pattern recognition workshops},
  pages={28--37},
  year={2019}
}

@inproceedings{sgd,
  title={Large-scale machine learning with stochastic gradient descent},
  author={Bottou, L{\'e}on},
  booktitle={Proceedings of COMPSTAT'2010: 19th International Conference on Computational StatisticsParis France, August 22-27, 2010 Keynote, Invited and Contributed Papers},
  pages={177--186},
  year={2010},
  organization={Springer}
}

@article{ISPRS,
  title={ISPRS semantic labeling contest},
  author={Rottensteiner, Franz and Sohn, Gunho and Gerke, Markus and Wegner, Jan D},
  journal={ISPRS: Leopoldsh{\"o}he, Germany},
  volume={1},
  number={4},
  pages={4},
  year={2014}
}

@inproceedings{deepglobe18,
  title={Deepglobe 2018: A challenge to parse the earth through satellite images},
  author={Demir, Ilke and Koperski, Krzysztof and Lindenbaum, David and Pang, Guan and Huang, Jing and Basu, Saikat and Hughes, Forest and Tuia, Devis and Raskar, Ramesh},
  booktitle={Proceedings of the IEEE conference on computer vision and pattern recognition workshops},
  pages={172--181},
  year={2018}
}

@inproceedings{RBC22,
  title={Rbc: Rectifying the biased context in continual semantic segmentation},
  author={Zhao, Hanbin and Yang, Fengyu and Fu, Xinghe and Li, Xi},
  booktitle={European Conference on Computer Vision},
  pages={55--72},
  year={2022},
  organization={Springer}
}

@article{DKD22,
  title={Decomposed knowledge distillation for class-incremental semantic segmentation},
  author={Baek, Donghyeon and Oh, Youngmin and Lee, Sanghoon and Lee, Junghyup and Ham, Bumsub},
  journal={Advances in neural information processing systems},
  volume={35},
  pages={10380--10392},
  year={2022}
}

@article{UCD22,
  title={Uncertainty-aware contrastive distillation for incremental semantic segmentation},
  author={Yang, Guanglei and Fini, Enrico and Xu, Dan and Rota, Paolo and Ding, Mingli and Nabi, Moin and Alameda-Pineda, Xavier and Ricci, Elisa},
  journal={IEEE Transactions on Pattern Analysis and Machine Intelligence},
  volume={45},
  number={2},
  pages={2567--2581},
  year={2022},
  publisher={IEEE}
}

@article{wang2023samrs,
  title={Samrs: Scaling-up remote sensing segmentation dataset with segment anything model},
  author={Wang, Di and Zhang, Jing and Du, Bo and Xu, Minqiang and Liu, Lin and Tao, Dacheng and Zhang, Liangpei},
  journal={Advances in Neural Information Processing Systems},
  volume={36},
  pages={8815--8827},
  year={2023}
}

@inproceedings{Comformer23,
  title={Comformer: Continual learning in semantic and panoptic segmentation},
  author={Cermelli, Fabio and Cord, Matthieu and Douillard, Arthur},
  booktitle={Proceedings of the IEEE/CVF Conference on Computer Vision and Pattern Recognition},
  pages={3010--3020},
  year={2023}
}

@article{zhu2025replay,
  title={Replay master: Automatic sample selection and effective memory utilization for continual semantic segmentation},
  author={Zhu, Lanyun and Chen, Tianrun and Yin, Jianxiong and See, Simon and Soh, De Wen and Liu, Jun},
  journal={IEEE Transactions on Pattern Analysis and Machine Intelligence},
  year={2025},
  publisher={IEEE}
}

@inproceedings{RECALL21,
  title={Recall: Replay-based continual learning in semantic segmentation},
  author={Maracani, Andrea and Michieli, Umberto and Toldo, Marco and Zanuttigh, Pietro},
  booktitle={Proceedings of the IEEE/CVF international conference on computer vision},
  pages={7026--7035},
  year={2021}
}

@inproceedings{MiB,
  title={Modeling the background for incremental learning in semantic segmentation},
  author={Cermelli, Fabio and Mancini, Massimiliano and Bulo, Samuel Rota and Ricci, Elisa and Caputo, Barbara},
  booktitle={Proceedings of the IEEE/CVF conference on computer vision and pattern recognition},
  pages={9233--9242},
  year={2020}
}

@inproceedings{xiaoreversible,
  title={Reversible Primitive--Composition Alignment for Continual Vision--Language Learning},
  author={Xiao, Canran and Xu, Tianxiang and Jiang, Yiyang and Gao, Haoyu and Wu, Yuhan and others},
  booktitle={The Fourteenth International Conference on Learning Representations},
  year = {2026}
}

@article{tan2025one,
  title={One-shot adaptation for cross-domain semantic segmentation in remote sensing images},
  author={Tan, Jiaojiao and Zhang, Haiwei and Yao, Ning and Yu, Qiang},
  journal={Pattern Recognition},
  volume={162},
  pages={111390},
  year={2025},
  publisher={Elsevier}
}

@article{zhang2026towards,
  title={Towards Open-Vocabulary Semantic Segmentation for Remote Sensing Images},
  author={Zhang, Da and Zeng, Mingmin and Li, Xuelong},
  journal={Pattern Recognition},
  pages={113120},
  year={2026},
  publisher={Elsevier}
}

@article{wang2025lmfnet,
  title={LMFNet: Lightweight Multimodal Fusion Network for high-resolution remote sensing image segmentation},
  author={Wang, Tong and Chen, Guanzhou and Zhang, Xiaodong and Liu, Chenxi and Wang, Jiaqi and Tan, Xiaoliang and Zhou, Wenlin and He, Chanjuan},
  journal={Pattern Recognition},
  volume={164},
  pages={111579},
  year={2025},
  publisher={Elsevier}
}

@article{ye2025lightweight,
  title={A lightweight multilevel multiscale dual-path fusion network for remote sensing semantic segmentation},
  author={Ye, Hong and Chang, Jiaming and Wang, Ke and Jia, Zhaohong and Sun, Wei and Li, Zhiwei},
  journal={Pattern Recognition},
  pages={112483},
  year={2025},
  publisher={Elsevier}
}

\clearpage

\appendix

\section{Supplementary Experimental Setup}
\label{app:exp-setup}

\subsection{Datasets and Incremental Protocols}
\label{app:datasets}

\vspace{1mm}
\noindent\textbf{DeepGlobe Land Cover.}
We use the DeepGlobe Land Cover dataset~\citep{deepglobe18} with 7 semantic classes (urban, agriculture, rangeland, forest, water, barren, and unknown/background).
Following standard practice, original tiles of size $2448\times2448$ are cut into non-overlapping patches of $512\times512$ pixels.
We use the official training split for incremental training and the official validation split for evaluation.
For class-incremental segmentation, we adopt a 4-step protocol:
(i) step~0 (base) contains roughly half of the classes (3 or 4 classes, depending on the split), and
(ii) steps~1--3 sequentially introduce the remaining classes in groups of size 1--2.
We keep the total number of classes and the per-step class cardinalities consistent with~\citet{GSMFRSDIL24,MiR25}; the exact assignment of classes to steps is fixed across all methods and is provided in the released code.

\vspace{1mm}
\noindent\textbf{ISPRS Vaihingen.}
The ISPRS Vaihingen dataset~\citep{ISPRS} consists of 33 high-resolution tiles with near-infrared, red and green channels and pixel-wise labels for 6 classes
(impervious surfaces, buildings, low vegetation, trees, cars, and clutter/background).
We follow~\citet{GSMFRSDIL24,D2LS25,MiR25} for pre-processing and train/validation splits:
16 images are used for training, 5 for validation, and the remaining 12 for testing.
All tiles are cropped into patches of $512\times512$ pixels with a stride of $256$.
We use the same 4-step class-incremental protocol as in~\citet{GSMFRSDIL24}, where the base step contains 3 classes and each incremental step adds 1 new foreground class.

\vspace{1mm}
\noindent\textbf{ISPRS Potsdam.}
The ISPRS Potsdam dataset~\citep{ISPRS} provides 38 tiles with RGB+IR channels and 6 semantic classes.
We follow the common protocol in~\citet{D2LS25,GSMFRSDIL24}: 24 images for training, 7 for validation, and 7 for testing.
We crop patches of size $512\times512$ with stride $256$.
The CISS protocol mirrors that of Vaihingen, with identical step-wise class groupings and number of incremental steps.

\vspace{1mm}
\noindent\textbf{iSAID.}
The iSAID dataset~\citep{ISAID} is a large-scale aerial image benchmark with high-resolution images and dense annotations. We adopt the semantic segmentation version used in~\citet{HIG22,STCLDRNet24}, which aggregates instance-level masks into semantic labels.
We use the official training/validation split and crop images into $512\times512$ patches with stride $256$.
Since iSAID is used to evaluate large-scale robustness rather than fine-grained incremental behavior, we employ the same 4-step CISS protocol as for DeepGlobe (same number of classes per step), while keeping the total class set and label mapping consistent with the official ontology.

\vspace{1mm}
\noindent\textbf{GCSS.}
The GCSS dataset~\citep{GCSS} consists of large-scale aerial imagery with densely annotated semantic labels.
We follow~\citet{HIG22,STCLDRNet24} for pre-processing and train/validation splits and adopt the same patch size ($512\times512$) and stride ($256$).
The incremental protocol again mirrors that of DeepGlobe and iSAID, with a 4-step schedule and the same relative class proportions per step.

\vspace{1mm}
\noindent\textbf{LoveDA (domain-incremental).}
LoveDA~\citep{LOVEDA} contains 5987 land-cover images annotated with 7 classes and split into urban and rural domains.
To study domain-/temporal shift, we follow the domain-incremental protocol introduced in GSMF-RS-DIL~\citep{GSMFRSDIL24}:
(i) step~0 trains on the urban subset,
(ii) step~1 trains on the rural subset, and
(iii) step~2 revisits a mixed domain containing both (simulating a later acquisition period with different domain proportions and appearance).
At each step we use the official training set for that domain configuration and evaluate on the full validation set.
We keep the class set fixed across steps, only the domain distribution changes over time.

\subsection{Pre-processing and Data Augmentation}
\label{app:preproc}

For all datasets, RGB channels are normalized with ImageNet-1K statistics (mean $[0.485, 0.456, 0.406]$, standard deviation $[0.229, 0.224, 0.225]$).
For Vaihingen and Potsdam, we map the NIR channel into the red channel following~\citet{D2LS25} when using ImageNet-pretrained encoders; alternatively, we also experimented with a $1\times1$ convolution-based channel adapter and observed negligible differences.

We use the following augmentations at training time:
\begin{itemize}
  \item Random horizontal and vertical flips with probability 0.5 each;
  \item Random scaling with factor uniformly sampled from $[0.5, 2.0]$, followed by random cropping to the target size $512\times512$;
  \item Random rotation by multiples of $90^\circ$;
  \item Color jitter on brightness, contrast, and saturation with maximum deviation $0.2$.
\end{itemize}

At evaluation time, we use a single center crop (or uniform tiling without overlap for large images) and do not apply test-time augmentation, in line with~\citet{D2LS25,GSMFRSDIL24}.

\subsection{Backbone and Segmentation Head}
\label{app:backbone}

All methods, including baselines and ProtoFlow, share the same segmentation backbone:

\begin{itemize}
  \item \textbf{Encoder:} HRNet-W48 pretrained on ImageNet-1K. We use an output stride of 4 and maintain multi-resolution feature fusion as in the original design.
  \item \textbf{Decoder:} We adopt the lightweight segmentation head of D2LS~\citep{D2LS25}, which consists of a $1\times1$ convolution to unify channel dimensions across branches, followed by two $3\times3$ convolutional layers with batch normalization and ReLU, and a final $1\times1$ classifier to output logits for $|\mathcal{C}^{\le t}|$ classes at step $t$. The logits are bilinearly upsampled to the original image resolution.
\end{itemize}

For a fair comparison, we fix the backbone and head architecture for all methods. Only the continual-learning components (distillation, memory, prototype regularization, etc.) differ.

\subsection{Implementation of ProtoFlow}
\label{app:protoflow-impl}
\vspace{1mm}
\noindent\textbf{Prototype estimator and bank.}
At each incremental step $t$, we compute class prototypes using Eq.~\eqref{eq:prototype_definition}.
In practice, we maintain a prototype bank $\{\mu_c^{(t)}\}_{c\in\mathcal{C}^{\le t}}$ updated with an exponential moving average (EMA) to reduce noise:
\begin{equation}
  \mu_c^{(t)} \leftarrow
  (1-\alpha)\,\mu_c^{(t)} + \alpha \,\tilde{\mu}_c^{(t)},
\end{equation}
where $\tilde{\mu}_c^{(t)}$ is the instantaneous mean feature computed on the current mini-batch (including memory samples if used), and $\alpha=0.1$.
We update the prototype bank once per training iteration.
Before computing $\mathcal{L}_{\text{flow}}$, $\mathcal{L}_{\text{curve}}$ and $\mathcal{L}_{\text{sep}}$, all prototypes are $L_2$-normalized, as described in the main text.

\vspace{1mm}
\noindent\textbf{ProtoFlow Field $F_{\phi}$.}
We implement $F_{\phi}$ as a two-layer MLP:
\begin{align*}
  h &= \mathrm{ReLU}\big( W_1 [\mu; e(\tau)] + b_1 \big), \\
  F_{\phi}(\mu,\tau) &= W_2 h + b_2,
\end{align*}
where $W_1 \in \mathbb{R}^{(d+d_\tau)\times 256}$, $W_2 \in \mathbb{R}^{256\times d}$, $b_1,b_2$ are biases, and $e(\tau)\in\mathbb{R}^{d_\tau}$ is a time encoding.
We set $d_\tau=16$ and use a sinusoidal positional encoding of the scalar time index normalized to $[0,1]$:
\[
  e(\tau_t)_k
  = \begin{cases}
      \sin\big( \omega_k \tilde{\tau}_t \big), & k \text{ odd},\\
      \cos\big( \omega_k \tilde{\tau}_t \big), & k \text{ even},
    \end{cases}
\]
where $\tilde{\tau}_t = (t - t_c)/(T_c - t_c)$ is the normalized step index for class $c$ and $\{\omega_k\}$ is a geometric frequency sequence.
Weights of $F_{\phi}$ are initialized with Kaiming initialization.
For the \emph{w/o Time conditioning} ablation, we simply drop $e(\tau)$ and feed $\mu$ through the same MLP.

\vspace{1mm}
\noindent\textbf{Loss weights and hyperparameters.}
For all datasets we use the same coefficient values unless otherwise noted:
\begin{itemize}
  \item Distillation temperature $T_{\text{dist}} = 2.0$;
  \item Distillation weight $\lambda_{\text{dist}} = 1.0$;
  \item Flow consistency weight $\lambda_{\text{flow}} = 1.0$;
  \item Curvature weight $\lambda_{\text{curve}} = 0.5$;
  \item Separation weight $\lambda_{\text{sep}} = 0.1$;
  \item Margin $m=0.5$ in Eq.~\eqref{eq:sep_step}.
\end{itemize}
We selected $(\lambda_{\text{curve}},\lambda_{\text{sep}})$ by a small grid search on DeepGlobe and kept them fixed for all other datasets.
For the \emph{w/o Curvature} and \emph{w/o Separation} ablations, we set the corresponding weights to zero while keeping all other hyperparameters unchanged.

\vspace{1mm}
\noindent\textbf{Memory buffer.}
For methods that rely on replay (MiB, APR, CoGaMiD, GSMF-RS-DIL, STCL-DRNet, MiR) and for ProtoFlow, we use a fixed memory budget of 20 images per class following MiB~\citep{MiB}.
We adopt herding selection on DeepGlobe, Vaihingen, and Potsdam, and random selection on iSAID, GCSS, and LoveDA for scalability.
The same memory size and selection strategy are used for ProtoFlow and baselines.

\subsection{Training Schedules and Optimization}
\label{app:training}

\paragraph{Optimizer and learning rate schedule.}
All models are trained with stochastic gradient descent (SGD) with momentum $0.9$ and weight decay $10^{-4}$.
We use a polynomial learning rate schedule with power $0.9$:
\begin{equation}
  \eta(i)
  = \eta_0 \left(1 - \frac{i}{I_{\text{max}}}\right)^{0.9},
\end{equation}
where $i$ is the current iteration and $I_{\text{max}}$ is the maximum number of iterations per step.
We use:
\begin{itemize}
  \item $\eta_0 = 0.01$, $I_{\text{max}} = 40\text{k}$ for DeepGlobe and LoveDA;
  \item $\eta_0 = 0.01$, $I_{\text{max}} = 30\text{k}$ for Vaihingen and Potsdam;
  \item $\eta_0 = 0.005$, $I_{\text{max}} = 60\text{k}$ for iSAID and GCSS (larger datasets).
\end{itemize}
For all experiments, we use a warm-up phase during the first $1\text{k}$ iterations where the learning rate increases linearly from $10^{-4}$ to $\eta_0$.

\vspace{1mm}
\noindent\textbf{Batch size and gradient clipping.}
Unless otherwise stated, we use a batch size of 8 on 4$\times$V100 GPUs (2 images per GPU) for DeepGlobe, Vaihingen, Potsdam, and LoveDA.
For iSAID and GCSS, we use a batch size of 4 due to memory constraints.
We apply gradient clipping with $\|\nabla\|_2 \le 1.0$ to stabilize training, especially when updating the ProtoFlow Field jointly with the encoder.

\vspace{1mm}
\noindent\textbf{Teacher model and distillation.}
At each incremental step $t>0$, we keep a copy of the model parameters $\theta^{\text{old}}$ from the end of step $t-1$ as the teacher.
The distillation loss in Eq.~\eqref{eq:dist_loss} is applied only on classes learned up to step $t-1$; newly introduced classes at step $t$ are not distilled.
We use logits from the penultimate feature map resolution (after the decoder but before final upsampling), and apply the softmax and KL divergence on every pixel in the mini-batch.

\subsection{Baseline Implementations and Hyperparameters}
\label{app:baselines}

We re-implement all baselines (MiB, DKD, LAG, APR, CoMBO, CoGaMiD, HIGCISS, GSMF-RS-DIL, MiSSNet, STCL-DRNet, MiR) in a unified codebase, using the same backbone, decoder, augmentation, and training schedules described above.
Whenever possible, we start from the official implementations and adapt them to HRNet-W48 while preserving their default hyperparameters (e.g., distillation weights, memory budget) and optimization tricks (e.g., class-balanced loss, auxiliary heads).
When this leads to instability or inferior performance compared to the numbers reported in their papers, we perform a small hyperparameter sweep on the validation set (learning rate $\eta_0 \in \{0.005, 0.01\}$, distillation weight $\lambda_{\text{dist}}\in\{0.5,1.0,2.0\}$) and report the best setting.

\subsection{Details of Ablation Variants}
\label{app:ablation-variants}

We detail the exact implementation of the ablation variants reported in Table~\ref{tab:ablation-protoflow}. 

\vspace{1mm}
\noindent\textbf{\ding{182} w/o ProtoFlow Field.}
This variant ablates the explicit vector field $F_{\phi}$ and the associated flow consistency loss $\mathcal{L}_{\text{flow}}^{(t)}$.
Concretely:

\begin{itemize}
  \item We remove the MLP $F_{\phi}$ entirely and do not predict prototype velocities $\mathbf{v}_c^{(k)} = F_{\phi}(\mu_c^{(k)},\tau_k)$.
  \item The forward-Euler prediction in Eq.~\eqref{eq:proto_euler},
  \begin{equation}
    \widehat{\mu}_c^{(k+1)} 
    = \mu_c^{(k)} + \Delta\tau_k F_{\phi}(\mu_c^{(k)},\tau_k),
  \end{equation}
  is not computed. Instead, only the empirical prototypes $\mu_c^{(k)}$ estimated from Eq.~\eqref{eq:prototype_definition} are used.
  \item The flow consistency loss in Eq.~\eqref{eq:flow_loss_step} is dropped by setting
  $
    \lambda_{\text{flow}} = 0,
    \qquad
    \mathcal{L}_{\text{flow}}^{(t)} \equiv 0.$
\end{itemize}

The curvature and separation terms are still computed on the observed prototype trajectories:
$
  \mathcal{L}_{\text{curve}}^{(t)} \text{ and } \mathcal{L}_{\text{sep}}^{(t)}
  \quad \text{are unchanged.}
$
The total loss thus reduces to
\begin{equation}
  \mathcal{L}^{(t)}_{\text{w/o-Field}} 
  = \mathcal{L}_{\text{seg}}^{(t)}
    + \lambda_{\text{dist}} \mathcal{L}_{\text{dist}}^{(t)}
    + \lambda_{\text{curve}} \mathcal{L}_{\text{curve}}^{(t)}
    + \lambda_{\text{sep}} \mathcal{L}_{\text{sep}}^{(t)}.
\end{equation}
This variant isolates the effect of explicitly modeling prototype dynamics via $F_{\phi}$ versus only regularizing the geometry of the empirical trajectory.

\vspace{1mm}
\noindent\textbf{\ding{183} w/o Curvature regularizer.}
This variant disables the curvature regularization that enforces low second-order differences of prototype trajectories.

\begin{itemize}
  \item We retain the ProtoFlow Field $F_{\phi}$ and the flow consistency loss $\mathcal{L}_{\text{flow}}^{(t)}$.
  \item The discrete curvature vectors
  \(
    \kappa_c^{(t)} 
    = \mu_c^{(t+1)} - 2\mu_c^{(t)} + \mu_c^{(t-1)}
  \)
  are no longer used in the loss.
  \item We set
  $
    \lambda_{\text{curve}} = 0,
    \qquad
    \mathcal{L}_{\text{curve}}^{(t)} \equiv 0.
  $
\end{itemize}

The total loss becomes
\begin{equation}
  \mathcal{L}^{(t)}_{\text{w/o-Curve}} 
  = \mathcal{L}_{\text{seg}}^{(t)}
    + \lambda_{\text{dist}} \mathcal{L}_{\text{dist}}^{(t)}
    + \lambda_{\text{flow}} \mathcal{L}_{\text{flow}}^{(t)}
    + \lambda_{\text{sep}} \mathcal{L}_{\text{sep}}^{(t)}.
\end{equation}
This configuration tests whether simply predicting prototype motion with $F_{\phi}$ (first-order dynamics) is sufficient, without penalizing high curvature in the induced trajectories.

\vspace{1mm}
\noindent\textbf{\ding{184} w/o Separation regularizer.}
This variant ablates the margin-based separation term that prevents prototypes of different classes from collapsing.

\begin{itemize}
  \item The ProtoFlow Field and flow consistency loss $\mathcal{L}_{\text{flow}}^{(t)}$ remain unchanged.
  \item The discrete curvature loss $\mathcal{L}_{\text{curve}}^{(t)}$ is kept.
  \item The separation loss
  \begin{equation}
    \mathcal{L}_{\text{sep}}^{(t)}
    = \sum_{\substack{c,c'\in \mathcal{C}^{\le t} \\ c'\neq c}}
      \left[
        m - \big\|\mu_c^{(t)} - \mu_{c'}^{(t)}\big\|_2
      \right]_+^2
  \end{equation}
  is disabled by setting
  $
    \lambda_{\text{sep}} = 0,
    \qquad
    \mathcal{L}_{\text{sep}}^{(t)} \equiv 0.
  $
\end{itemize}

The resulting loss is
\begin{equation}
  \mathcal{L}^{(t)}_{\text{w/o-Sep}} 
  = \mathcal{L}_{\text{seg}}^{(t)}
    + \lambda_{\text{dist}} \mathcal{L}_{\text{dist}}^{(t)}
    + \lambda_{\text{flow}} \mathcal{L}_{\text{flow}}^{(t)}
    + \lambda_{\text{curve}} \mathcal{L}_{\text{curve}}^{(t)}.
\end{equation}
This variant isolates the benefit of enforcing class-wise margins along the prototype flow, while still constraining curvature.

\vspace{1mm}
\noindent\textbf{\ding{185} w/o Time conditioning.}
This variant tests whether the ProtoFlow Field genuinely exploits temporal information, or simply learns a generic mapping that ignores time.

\begin{itemize}
  \item We keep the overall ProtoFlow pipeline identical, including $\mathcal{L}_{\text{flow}}^{(t)}$, $\mathcal{L}_{\text{curve}}^{(t)}$, and $\mathcal{L}_{\text{sep}}^{(t)}$.
  \item The only change is in the parameterization of $F_{\phi}$:
  instead of $F_{\phi} : \mathbb{R}^d \times \mathbb{R} \to \mathbb{R}^d$ with input $(\mu, \tau)$, we drop the time encoding and use
  $
    F_{\phi}^{\text{no-time}} : \mathbb{R}^d \to \mathbb{R}^d,
    \quad
    F_{\phi}^{\text{no-time}}(\mu)
    = W_2 \,\mathrm{ReLU}(W_1 \mu + b_1) + b_2.
  $
  \item All architectural hyperparameters (hidden dimension, initialization, etc.) and loss weights are kept the same as in the full model.
\end{itemize}

The objective remains
\begin{equation}
  \begin{split}
    \mathcal{L}^{(t)}_{\text{w/o-Time}} 
    = \mathcal{L}_{\text{seg}}^{(t)}
      + \lambda_{\text{dist}} \mathcal{L}_{\text{dist}}^{(t)}
      + \lambda_{\text{flow}} \mathcal{L}_{\text{flow}}^{(t)} \\
    \quad + \lambda_{\text{curve}} \mathcal{L}_{\text{curve}}^{(t)}
      + \lambda_{\text{sep}} \mathcal{L}_{\text{sep}}^{(t)},
  \end{split}
\end{equation}
but the flow field cannot condition its predictions on acquisition time.
This ablation directly probes the contribution of time-aware modeling in $F_{\phi}$.

\vspace{1mm}
\noindent\textbf{\ding{186} w/o Prototype norm.}
Finally, this variant removes the $L_2$ normalization of prototypes before computing prototype-based losses.

\begin{itemize}
  \item In the full model, we normalize all prototypes before using them in $\mathcal{L}_{\text{flow}}^{(t)}$, $\mathcal{L}_{\text{curve}}^{(t)}$, and $\mathcal{L}_{\text{sep}}^{(t)}$:
  \begin{equation}
    \tilde{\mu}_c^{(k)}
    = \frac{\mu_c^{(k)}}{\|\mu_c^{(k)}\|_2},
    \qquad
    \tilde{\widehat{\mu}}_c^{(k)}
    = \frac{\widehat{\mu}_c^{(k)}}{\|\widehat{\mu}_c^{(k)}\|_2}.
  \end{equation}
  All distances and differences in prototype space are computed using $\tilde{\mu}$ and $\tilde{\widehat{\mu}}$.
  \item In the w/o Prototype norm. variant, we instead use the raw prototypes without normalization:
\begin{align}
  \mathcal{L}_{\text{flow}}^{(t)}
  &= \sum_{c \in \mathcal{C}^{\le t-1}}
     \big\| \widehat{\mu}_c^{(t)} - \mu_c^{(t)} \big\|_2^2, \\
  \mathcal{L}_{\text{curve}}^{(t)}
  &= \sum_{c \in \mathcal{C}^{\le t}}
     \mathbf{1}[t \ge t_c+1] \nonumber \\
  &\quad \cdot \big\|
        \mu_c^{(t+1)} - 2\mu_c^{(t)} + \mu_c^{(t-1)}
      \big\|_2^2, \\
  \mathcal{L}_{\text{sep}}^{(t)}
  &= \sum_{\substack{c,c'\in \mathcal{C}^{\le t}\\c'\neq c}}
     \left[
       m - \big\| \mu_c^{(t)} - \mu_{c'}^{(t)} \big\|_2
     \right]_+^2.
\end{align}
  \item All other components of the pipeline (memory, encoder, $F_{\phi}$) are identical to the full ProtoFlow.
\end{itemize}
This variant reveals whether constraining prototype dynamics on the unit sphere yields better stability than operating in the unconstrained Euclidean feature space.

\section{Theoretical Analysis}
\label{sec:theory}

We now give a first-step theoretical analysis of how prototype trajectory geometry controls forgetting and dynamic regret in a stylized setting.
The goal is to make explicit the link between (i)~the curvature of prototype trajectories and (ii)~the growth of classification error over time.

\subsection{A Stylized Prototype-Gaussian Model}

We focus on a single foreground class $c$ whose prototype drifts over time,
while all other class prototypes remain fixed.\footnote{
Extending the analysis to multiple drifting classes is possible but leads
to heavier notation without adding new insight. Here we isolate the
effect of prototype drift for one class.
}
This captures the typical situation in continual learning where an old class
can be forgotten because its representation moves in feature space
under non-stationary training.

\paragraph{Time index and prototypes.}
Fix a class $c$ and a horizon $T \ge 2$.
For simplicity we re-index time so that $c$ first appears at step $t=0$.
Let
\begin{equation}
  \mu^{(t)} \in \mathbb{R}^d, \quad t = 0,1,\dots,T,
\end{equation}
denote the prototype of class $c$ at step $t$.
All other classes $c' \neq c$ are assumed to have fixed prototypes
$\mu_{c'} \in \mathbb{R}^d$ independent of $t$.

\paragraph{Gaussian feature model.}
We adopt the following standard parametric model for the encoder features.

\begin{assumption}
\label{assump:gaussian}
For each step $t$ and class $y \in \{c\} \cup \mathcal{C}_{\text{other}}$,
the encoder feature $z \in \mathbb{R}^d$ of a pixel sampled from class $y$
at step $t$ satisfies
\begin{equation}
  \begin{split}
    z \,\big|\, y=c, t \sim \mathcal{N}\big(\mu^{(t)}, \sigma^2 I_d\big), \\
    z \,\big|\, y=c', t \sim \mathcal{N}\big(\mu_{c'}, \sigma^2 I_d\big),
  \end{split}
\end{equation}
for all $c' \neq c$, where $\sigma > 0$ is fixed and $I_d$ is the
$d\times d$ identity matrix.
\end{assumption}

\paragraph{Nearest-prototype classifier and margin.}
At step $t$ the classifier $h_t$ is the nearest-prototype rule in feature space:
\begin{equation}
  h_t(z) = \argmin_{y \in \{c\} \cup \mathcal{C}_{\text{other}}}
  \big\| z - \mu_y^{(t)} \big\|,
\end{equation}
where we set $\mu_c^{(t)} := \mu^{(t)}$ and $\mu_{c'}^{(t)} := \mu_{c'}$
for all $c' \neq c$.
We write $K = 1 + |\mathcal{C}_{\text{other}}|$ for the number of classes.

The margin of class $c$ at step $t$ is
\begin{equation}
  \gamma^{(t)} 
  := \min_{c' \neq c} 
      \big\| \mu^{(t)} - \mu_{c'} \big\|.
  \label{eq:margin_def}
\end{equation}
We assume that the margin never collapses completely.

\begin{assumption}[Uniform separation]
\label{assump:margin}
There exists $\gamma_{\min}>0$ such that $\gamma^{(t)} \ge \gamma_{\min}$
for all $t=0,\dots,T$.
\end{assumption}

\paragraph{Per-class risk and forgetting.}
We focus on the population misclassification probability for class $c$ at
step $t$:
\begin{equation}
  R^{(t)} 
  := \Pr\big( h_t(z) \neq c \,\big|\, y=c, t \big).
  \label{eq:risk_def}
\end{equation}
The final forgetting of class $c$ at horizon $T$ is defined as
\begin{equation}
  F_c(T)
  := R^{(T)} - \min_{0 \le s \le T} R^{(s)}.
  \label{eq:forgetting_def}
\end{equation}
Equivalently, let $t^\star \in \argmin_{0 \le s \le T} R^{(s)}$ be a step
where class $c$ is best remembered, then
$F_c(T) = R^{(T)} - R^{(t^\star)}$.

We also define a per-class dynamic regret with respect to the best
static representation for class $c$:

\begin{definition}
\label{def:dyn_regret}
Let $R_c^\star := \min_{0 \le s \le T} R^{(s)}$ and define
\begin{equation}
  \mathrm{Reg}_c(T)
  := \sum_{t=0}^T \big( R^{(t)} - R_c^\star \big).
  \label{eq:dyn_regret_def}
\end{equation}
\end{definition}

By construction $R^{(t)} \ge R_c^\star$ for all $t$, hence
$\mathrm{Reg}_c(T)$ measures the cumulative regret incurred by
the time-varying prototype trajectory relative to the best
fixed prototype for class $c$ in hindsight.
Note that $F_c(T)$ upper-bounds the per-step excess term in
\eqref{eq:dyn_regret_def}, so
\begin{equation}
  \mathrm{Reg}_c(T)
  \;\le\; (T+1) \, F_c(T).
  \label{eq:regret_forgetting_relation}
\end{equation}

\subsection{From Margin to Risk: A Gaussian Tail Bound}

We first relate the per-class risk $R^{(t)}$ to the margin $\gamma^{(t)}$.

\begin{lemma}[Gaussian margin bound]
\label{lem:gauss_margin_bound}
Under Assumption~\ref{assump:gaussian}, for each step $t$,
the misclassification probability of class $c$ satisfies
\begin{equation}
  R^{(t)}
  \;\le\;
  g\big(\gamma^{(t)}\big)
  := (K-1) \exp\!\left( 
      - \frac{ \big(\gamma^{(t)}\big)^2 }{ 8 \sigma^2 }
    \right).
  \label{eq:gauss_margin_bound}
\end{equation}
\end{lemma}

\begin{proof}
Fix a step $t$ and abbreviate $\mu := \mu^{(t)}$.
For each competitor class $c' \neq c$,
let $\delta_{c'} := \| \mu_{c'} - \mu \|$ and define
the unit direction
\begin{equation}
  u_{c'} := \frac{ \mu_{c'} - \mu }{ \delta_{c'} }.
\end{equation}
The decision boundary between $c$ and $c'$ under the nearest-prototype rule
is the hyperplane
\begin{equation}
  H_{c'} 
  = \big\{ z \in \mathbb{R}^d :
      \|z - \mu\| = \|z - \mu_{c'}\|
    \big\},
\end{equation}
which can be rewritten as
\begin{equation}
  \big\langle z - \tfrac12(\mu + \mu_{c'}), u_{c'} \big\rangle = 0.
\end{equation}
Under Assumption~\ref{assump:gaussian},
for a sample $z$ from class $c$ at step $t$ we have
$z \sim \mathcal{N}(\mu, \sigma^2 I_d)$.
The scalar projection onto $u_{c'}$ is
\begin{equation}
  s_{c'} 
  := \left\langle z - \tfrac12(\mu + \mu_{c'}), u_{c'} \right\rangle
  = \left\langle z - \mu, u_{c'} \right\rangle - \tfrac12 \delta_{c'}.
\end{equation}
Since $z-\mu \sim \mathcal{N}(0, \sigma^2 I_d)$ and $u_{c'}$ is unit-norm,
we have
\begin{equation}
  \left\langle z-\mu, u_{c'} \right\rangle \sim \mathcal{N}(0, \sigma^2),
\end{equation}
so
\begin{equation}
  s_{c'} \sim \mathcal{N}\Big(-\tfrac12 \delta_{c'},\, \sigma^2\Big).
\end{equation}
Now $z$ is misclassified as $c'$ iff $z$ lies on the $c'$ side of $H_{c'}$,
i.e., $s_{c'} > 0$. Therefore
\begin{equation}
  \begin{split}
    \Pr\big( h_t(z)=c' \,\big|\, y=c, t \big)
    = \Pr\big( s_{c'} > 0 \big) \\
    = \Pr\!\left(
          \mathcal{N}\Big(-\tfrac12 \delta_{c'},\, \sigma^2\Big) > 0
        \right).
  \end{split}
\end{equation}
Let $Z \sim \mathcal{N}(0,1)$. Then
\begin{equation}
  \Pr\big( s_{c'} > 0 \big)
  = \Pr\!\left(
      Z > \frac{ \delta_{c'} }{ 2 \sigma }
    \right).
\end{equation}
Using the standard Gaussian tail bound
$\Pr(Z>a) \le \exp(-a^2/2)$ for all $a>0$, we obtain
\begin{equation}
  \Pr\big( h_t(z)=c' \,\big|\, y=c, t \big)
  \le \exp\!\left(
    - \frac{ \delta_{c'}^2 }{ 8 \sigma^2 }
  \right).
\end{equation}
By the union bound over the $K-1$ competing classes,
\begin{align}
  R^{(t)}
  &= \Pr\big( h_t(z) \neq c \,\big|\, y=c, t \big)                           \\
  &\le \sum_{c' \neq c}
      \Pr\big( h_t(z)=c' \,\big|\, y=c, t \big)                              \\
  &\le \sum_{c' \neq c}
      \exp\!\left(
        - \frac{ \delta_{c'}^2 }{ 8 \sigma^2 }
      \right).
\end{align}
Since $\delta_{c'} \ge \gamma^{(t)}$ for all $c' \neq c$, we have
\begin{equation}
  R^{(t)}
  \le (K-1)
      \exp\!\left(
        - \frac{ \big(\gamma^{(t)}\big)^2 }{ 8 \sigma^2 }
      \right)
  = g\big(\gamma^{(t)}\big),
\end{equation}
which is exactly \eqref{eq:gauss_margin_bound}. \hfill$\blacksquare$
\end{proof}

Lemma~\ref{lem:gauss_margin_bound} shows that the risk $R^{(t)}$ is
controlled by the margin $\gamma^{(t)}$.
We next show that $g(\cdot)$ is Lipschitz on the margin range of interest.

\begin{lemma}[Lipschitz continuity of $g$]
\label{lem:g_lipschitz}
Assume $\gamma^{(t)} \ge \gamma_{\min} > 0$ for all $t$.
Define
\begin{equation}
  L_g
  :=
  \begin{cases}
    \displaystyle
    (K-1)\,\frac{1}{2\sigma}\, e^{-1/2},
    & \text{if } \gamma_{\min} \le 2\sigma, \\[1ex]
    \displaystyle
    (K-1)\,\frac{\gamma_{\min}}{4\sigma^2}\,
      \exp\!\Big(-\tfrac{\gamma_{\min}^2}{8\sigma^2}\Big),
    & \text{if } \gamma_{\min} > 2\sigma.
  \end{cases}
  \label{eq:Lg_def}
\end{equation}
Then for all $\gamma_1,\gamma_2 \ge \gamma_{\min}$,
\begin{equation}
  \big| g(\gamma_2) - g(\gamma_1) \big|
  \le L_g \, |\gamma_2 - \gamma_1|.
  \label{eq:g_lipschitz}
\end{equation}
\end{lemma}

\begin{proof}
We compute the derivative of $g$:
\begin{equation}
  \begin{split}
    g(\gamma)
    = (K-1) \exp\!\left( -\frac{\gamma^2}{8\sigma^2} \right), \\
    g'(\gamma)
    = -(K-1)\,\frac{\gamma}{4\sigma^2}
       \exp\!\left( -\frac{\gamma^2}{8\sigma^2} \right).
  \end{split}
\end{equation}
The magnitude of the derivative is
\begin{equation}
  |g'(\gamma)|
  = (K-1)\,\frac{\gamma}{4\sigma^2}
     \exp\!\left( -\frac{\gamma^2}{8\sigma^2} \right)
  =: (K-1)\,h(\gamma).
\end{equation}
To bound $\sup_{\gamma \ge \gamma_{\min}} h(\gamma)$, consider
\begin{equation}
  h(\gamma)
  = \frac{\gamma}{4\sigma^2}
    \exp\!\left( -\frac{\gamma^2}{8\sigma^2} \right).
\end{equation}
Differentiating,
\begin{align}
  h'(\gamma)
  &= \frac{1}{4\sigma^2}
     \exp\!\left( -\frac{\gamma^2}{8\sigma^2} \right)
     - \frac{\gamma^2}{32\sigma^4}
       \exp\!\left( -\frac{\gamma^2}{8\sigma^2} \right)                \\
  &= \frac{1}{4\sigma^2}
     \exp\!\left( -\frac{\gamma^2}{8\sigma^2} \right)
     \left( 1 - \frac{\gamma^2}{4\sigma^2} \right).
\end{align}
Thus $h'(\gamma)>0$ for $\gamma < 2\sigma$ and $h'(\gamma)<0$ for
$\gamma > 2\sigma$, so $h(\gamma)$ attains its global maximum at
$\gamma = 2\sigma$, with
\begin{equation}
  h(2\sigma)
  = \frac{2\sigma}{4\sigma^2}\, e^{-1/2}
  = \frac{1}{2\sigma}\, e^{-1/2}.
\end{equation}

\noindent\textbf{Case 1: $\gamma_{\min} \le 2\sigma$.}
Then $\sup_{\gamma \ge \gamma_{\min}} h(\gamma) = h(2\sigma)$ and hence
\begin{equation}
  \sup_{\gamma \ge \gamma_{\min}} |g'(\gamma)|
  = (K-1)\,h(2\sigma)
  = (K-1)\,\frac{1}{2\sigma}\, e^{-1/2}.
\end{equation}

\noindent\textbf{Case 2: $\gamma_{\min} > 2\sigma$.}
In this case $h(\gamma)$ is decreasing on $[\gamma_{\min},\infty)$, so
\begin{equation}
  \sup_{\gamma \ge \gamma_{\min}} h(\gamma)
  = h(\gamma_{\min})
  = \frac{\gamma_{\min}}{4\sigma^2}
    \exp\!\left( -\frac{\gamma_{\min}^2}{8\sigma^2} \right).
\end{equation}
Thus
\begin{equation}
  \sup_{\gamma \ge \gamma_{\min}} |g'(\gamma)|
  = (K-1)\,\frac{\gamma_{\min}}{4\sigma^2}
    \exp\!\left( -\frac{\gamma_{\min}^2}{8\sigma^2} \right).
\end{equation}

Combining the two cases and applying the mean value theorem yields
\eqref{eq:g_lipschitz}. \hfill$\blacksquare$
\end{proof}

\subsection{From Prototype Trajectories to Margin Variation}

We now connect the evolution of the margin to the geometry of the
prototype trajectory $\{\mu^{(t)}\}_{t=0}^T$.

\paragraph{First and second differences.}
Define the first-order differences (``velocities'') and 
second-order differences (``curvatures'') of the prototype trajectory:
\begin{align}
  v^{(t)} 
  &:= \mu^{(t)} - \mu^{(t-1)}, \quad t = 1,\dots,T, \label{eq:v_def} \\
  \kappa^{(t)} 
  &:= \mu^{(t+1)} - 2\mu^{(t)} + \mu^{(t-1)} \nonumber \\
  &= v^{(t+1)} - v^{(t)},\quad t = 1,\dots,T-1. \label{eq:kappa_def}
\end{align}
For notational simplicity we will write sums over $t=1,\dots,T$ understanding
that $\kappa^{(T)}$ is not defined and simply omitting that term.

\paragraph{Path length and curvature energy.}
We quantify the amount of prototype movement by the path length
\begin{equation}
  S
  := \sum_{t=1}^T \big\| v^{(t)} \big\|,
  \label{eq:path_length}
\end{equation}
and the ``curvature energy''
\begin{equation}
  \mathcal{K}
  := \big\| v^{(1)} \big\|^2 
   + \sum_{t=1}^{T-1} \big\| \kappa^{(t)} \big\|^2.
  \label{eq:curvature_energy}
\end{equation}
Note that $\mathcal{K}$ differs from the ProtoFlow curvature loss
by a boundary term $\|v^{(1)}\|^2$. In practice this term is small and
can be absorbed into the regularizer.

\begin{lemma}[Margin is Lipschitz in path length]
\label{lem:margin_path}
Fix any two steps $0 \le s < t \le T$ and let
\begin{equation}
  S(s,t)
  := \sum_{u=s+1}^t \big\| v^{(u)} \big\|
\end{equation}
be the length of the prototype path between $s$ and $t$.
Then the margin satisfies
\begin{equation}
  \big| \gamma^{(t)} - \gamma^{(s)} \big|
  \le S(s,t).
  \label{eq:margin_path_bound}
\end{equation}
\end{lemma}

\begin{proof}
Fix any competitor class $c' \neq c$ and define
\begin{equation}
  d_{c'}^{(u)} := \big\| \mu^{(u)} - \mu_{c'} \big\|,
  \quad u = 0,1,\dots,T.
\end{equation}
By the triangle inequality,
\begin{align}
  \big\| \mu^{(t)} - \mu_{c'} \big\|
  &= \big\| \mu^{(s)} - \mu_{c'} + (\mu^{(t)} - \mu^{(s)}) \big\|   \\
  &\le \big\| \mu^{(s)} - \mu_{c'} \big\|
     + \big\| \mu^{(t)} - \mu^{(s)} \big\|.
\end{align}
Using telescoping,
\begin{equation}
  \mu^{(t)} - \mu^{(s)}
  = \sum_{u=s+1}^t v^{(u)},
\end{equation}
hence
\begin{equation}
  \big\| \mu^{(t)} - \mu^{(s)} \big\|
  \le \sum_{u=s+1}^t \big\| v^{(u)} \big\|
  = S(s,t).
\end{equation}
Therefore
\begin{equation}
  d_{c'}^{(t)}
  \le d_{c'}^{(s)} + S(s,t).
\end{equation}
Similarly, swapping the roles of $s$ and $t$, we obtain
\begin{equation}
  d_{c'}^{(s)}
  \le d_{c'}^{(t)} + S(s,t),
\end{equation}
so
\begin{equation}
  \big| d_{c'}^{(t)} - d_{c'}^{(s)} \big|
  \le S(s,t).
\end{equation}

Now recall that the margin is the minimum pairwise distance:
\begin{equation}
  \gamma^{(u)} 
  = \min_{c' \neq c} d_{c'}^{(u)},
  \quad u \in \{s,t\}.
\end{equation}
For any $c' \neq c$ we have
\begin{equation}
  d_{c'}^{(t)}
  \ge d_{c'}^{(s)} - S(s,t)
  \ge \gamma^{(s)} - S(s,t),
\end{equation}
hence
\begin{equation}
  \gamma^{(t)}
  = \min_{c' \neq c} d_{c'}^{(t)}
  \ge \gamma^{(s)} - S(s,t).
\end{equation}
Conversely, choosing $c'^\star$ such that 
$d_{c'^\star}^{(s)} = \gamma^{(s)}$, we have
\begin{equation}
  \gamma^{(t)}
  \le d_{c'^\star}^{(t)}
  \le d_{c'^\star}^{(s)} + S(s,t)
  = \gamma^{(s)} + S(s,t).
\end{equation}
Combining the two inequalities yields
\begin{equation}
  \gamma^{(s)} - S(s,t)
  \le \gamma^{(t)}
  \le \gamma^{(s)} + S(s,t),
\end{equation}
which is equivalent to \eqref{eq:margin_path_bound}. \hfill$\blacksquare$
\end{proof}

Next we show that the path length $S$ can itself be controlled by
the curvature energy $\mathcal{K}$.
This is a discrete Sobolev-type inequality.

\begin{lemma}
\label{lem:path_curvature}
Let $v^{(t)}$ and $\kappa^{(t)}$ be defined as in
\eqref{eq:v_def}--\eqref{eq:kappa_def} and let $\mathcal{K}$ be
as in \eqref{eq:curvature_energy}.
Then
\begin{equation}
  \sum_{t=1}^T \big\| v^{(t)} \big\|^2
  \;\le\;
  \frac{T(T+1)}{2} \, \mathcal{K},
  \label{eq:v2_curvature_bound}
\end{equation}
and consequently the path length satisfies
\begin{equation}
  S
  = \sum_{t=1}^T \big\| v^{(t)} \big\|
  \;\le\;
  T \sqrt{\frac{T+1}{2}} \,
  \sqrt{\mathcal{K}}.
  \label{eq:path_curvature_bound}
\end{equation}
\end{lemma}

\begin{proof}
Define the sequence of ``increments of increments''
\begin{gather}
  b^{(1)} := v^{(1)}, \\
  b^{(t)} := v^{(t)} - v^{(t-1)} = \kappa^{(t-1)}
  \quad \text{for } t \ge 2.
\end{gather}
Then for each $t \ge 1$ we can write $v^{(t)}$ as a telescoping sum
of the $b^{(u)}$:
\begin{equation}
  v^{(t)}
  = \sum_{u=1}^t b^{(u)}.
\end{equation}
By the Cauchy--Schwarz inequality,
\begin{equation}
  \big\| v^{(t)} \big\|^2
  = \left\| \sum_{u=1}^t b^{(u)} \right\|^2
  \le t \sum_{u=1}^t \big\| b^{(u)} \big\|^2.
\end{equation}
Summing over $t=1,\dots,T$ gives
\begin{align}
  \sum_{t=1}^T \big\| v^{(t)} \big\|^2
  &\le \sum_{t=1}^T
       t \sum_{u=1}^t \big\| b^{(u)} \big\|^2               \\
  &=   \sum_{u=1}^T \big\| b^{(u)} \big\|^2
       \sum_{t=u}^T t                                      \\
  &=   \sum_{u=1}^T \big\| b^{(u)} \big\|^2
       \frac{T(T+1) - (u-1)u}{2}                           \\
  &\le \frac{T(T+1)}{2}
       \sum_{u=1}^T \big\| b^{(u)} \big\|^2.
\end{align}
By construction,
\begin{equation}
  \sum_{u=1}^T \big\| b^{(u)} \big\|^2
  = \big\| v^{(1)} \big\|^2
    + \sum_{t=1}^{T-1} \big\| \kappa^{(t)} \big\|^2
  = \mathcal{K},
\end{equation}
so we obtain \eqref{eq:v2_curvature_bound}.

To derive \eqref{eq:path_curvature_bound}, apply Cauchy--Schwarz:
\begin{equation}
  \begin{split}
    S = \sum_{t=1}^T \big\| v^{(t)} \big\|
    \le \sqrt{T} \left( \sum_{t=1}^T \big\| v^{(t)} \big\|^2 \right)^{1/2} \\
    \le \sqrt{T} \left( \frac{T(T+1)}{2} \, \mathcal{K} \right)^{1/2},
  \end{split}
\end{equation}
which simplifies to \eqref{eq:path_curvature_bound}. \hfill$\blacksquare$
\end{proof}

\subsection{Curvature-Controlled Forgetting and Dynamic Regret}

We are now ready to state the main result: a curvature-controlled
upper bound on forgetting and dynamic regret.

\begin{theorem}[Curvature control of forgetting]
\label{thm:forgetting_bound}
Suppose Assumptions~\ref{assump:gaussian} and~\ref{assump:margin} hold.
Let $t^\star \in \argmin_{0 \le s \le T} R^{(s)}$ be a time where
class $c$ is best remembered, and let $\mathcal{K}$ be the curvature
energy in \eqref{eq:curvature_energy}.
Then the final forgetting $F_c(T)$ defined in
\eqref{eq:forgetting_def} satisfies
\begin{equation}
  F_c(T)
  \;\le\;
  L_g \, T \sqrt{\frac{T+1}{2}} \,
  \sqrt{\mathcal{K}},
  \label{eq:forgetting_bound}
\end{equation}
where $L_g$ is the Lipschitz constant of $g$ given in
\eqref{eq:Lg_def}.
\end{theorem}

\begin{proof}
By definition of $t^\star$ we have
\begin{equation}
  F_c(T)
  = R^{(T)} - R^{(t^\star)}.
\end{equation}
Using Lemma~\ref{lem:gauss_margin_bound},
\begin{equation}
  R^{(T)}
  \le g\big(\gamma^{(T)}\big),
  \quad
  R^{(t^\star)}
  \ge 0,
\end{equation}
so
\begin{equation}
  F_c(T)
  = R^{(T)} - R^{(t^\star)}
  \le g\big(\gamma^{(T)}\big) - R^{(t^\star)}.
  \label{eq:F_step1}
\end{equation}
We now upper-bound $g(\gamma^{(T)}) - R^{(t^\star)}$ by two steps.

\smallskip
\noindent\textbf{Step 1: Bounding risk difference by margin difference.}
Again by Lemma~\ref{lem:gauss_margin_bound},
\begin{equation}
  R^{(t^\star)}
  \le g\big(\gamma^{(t^\star)}\big).
\end{equation}
Combining with \eqref{eq:F_step1} yields
\begin{equation}
  F_c(T)
  \le g\big(\gamma^{(T)}\big) - g\big(\gamma^{(t^\star)}\big).
\end{equation}
Applying the Lipschitz property of $g$ from Lemma~\ref{lem:g_lipschitz},
we obtain
\begin{equation}
  F_c(T)
  \le L_g \, \big| \gamma^{(T)} - \gamma^{(t^\star)} \big|.
  \label{eq:F_step2}
\end{equation}

\smallskip
\noindent\textbf{Step 2: Bounding margin difference by path length.}
Using Lemma~\ref{lem:margin_path} with $(s,t)=(t^\star,T)$,
\begin{equation}
  \big| \gamma^{(T)} - \gamma^{(t^\star)} \big|
  \le S(t^\star,T)
  := \sum_{u=t^\star+1}^T \big\| v^{(u)} \big\|.
\end{equation}
Clearly $S(t^\star,T) \le S$, where $S$ is the full path length
\eqref{eq:path_length}. Therefore
$
  \big| \gamma^{(T)} - \gamma^{(t^\star)} \big|
  \le S.
$
Substituting into \eqref{eq:F_step2} gives
\begin{equation}
  F_c(T)
  \le L_g \, S.
  \label{eq:F_step3}
\end{equation}

\smallskip
\noindent\textbf{Step 3: Bounding path length by curvature energy.}
Finally, applying the path-curvature inequality
\eqref{eq:path_curvature_bound} from Lemma~\ref{lem:path_curvature},
we obtain
\begin{equation}
  S
  \le T \sqrt{\frac{T+1}{2}} \, \sqrt{\mathcal{K}}.
\end{equation}
Substituting this into \eqref{eq:F_step3} yields
\begin{equation}
  F_c(T)
  \le L_g \, T \sqrt{\frac{T+1}{2}} \, \sqrt{\mathcal{K}},
\end{equation}
which is exactly \eqref{eq:forgetting_bound}. \hfill$\blacksquare$
\end{proof}

Theorem~\ref{thm:forgetting_bound} shows that, in this stylized model,
the final forgetting of class $c$ grows at most linearly with the
time horizon and only as the square root of the curvature energy
$\mathcal{K}$.  In particular, if the curvature energy is controlled
by the ProtoFlow regularizer,
\begin{equation}
  \mathcal{K}
  = \mathcal{O}\!\left( \frac{1}{\lambda_{\text{curve}}} \right),
\end{equation}
then increasing $\lambda_{\text{curve}}$ tightens the bound on
$F_c(T)$, implying less forgetting.

As a direct consequence of \eqref{eq:regret_forgetting_relation}, we
obtain a dynamic regret bound.

\begin{corollary}[Per-class dynamic regret bound]
\label{cor:dyn_regret_bound}
Under the assumptions of Theorem~\ref{thm:forgetting_bound},
the per-class dynamic regret $\mathrm{Reg}_c(T)$ defined in
\eqref{eq:dyn_regret_def} satisfies
\begin{equation}
  \begin{split}
    \mathrm{Reg}_c(T)
    \;\le\;
    (T+1)\,L_g \, T \sqrt{\frac{T+1}{2}} \,
    \sqrt{\mathcal{K}} \\
    \;=\;
    \mathcal{O}\!\big( T^{5/2} \sqrt{\mathcal{K}} \big).
    \label{eq:dyn_regret_bound}
  \end{split}
\end{equation}
\end{corollary}

\begin{proof}
By definition $R_c^\star = \min_{0 \le s \le T} R^{(s)} = R^{(t^\star)}$ for
some $t^\star$.
For each $t$ we have $R^{(t)} - R_c^\star \le F_c(T)$, since
$F_c(T) = \max_{0 \le s \le T} (R^{(s)} - R_c^\star)$ by
\eqref{eq:forgetting_def}.
Therefore
\begin{equation}
  \mathrm{Reg}_c(T)
  = \sum_{t=0}^T \big( R^{(t)} - R_c^\star \big)
  \le (T+1) \, F_c(T).
\end{equation}
Substituting bound \eqref{eq:forgetting_bound} on $F_c(T)$ yields
\eqref{eq:dyn_regret_bound}.\hfill$\blacksquare$
\end{proof}

Theorem~\ref{thm:forgetting_bound} and
Corollary~\ref{cor:dyn_regret_bound}
 connect three quantities:
\begin{itemize}
  \item the \emph{geometric regularity} of prototype trajectories
        through the curvature energy $\mathcal{K}$;
  \item the \emph{final forgetting} $F_c(T)$ of an old class; and
  \item the \emph{dynamic regret} $\mathrm{Reg}_c(T)$ of using
        a time-varying prototype instead of the best static one.
\end{itemize}

In particular, if ProtoFlow successfully keeps prototype trajectories
low-curvature (small $\mathcal{K}$) while preserving a positive margin
$\gamma_{\min}$, then both forgetting and dynamic regret are provably
bounded.  This provides a theoretical justification for explicitly
shaping prototype flows with curvature regularization, as implemented
in our method.

\section{Additional Experimental Results}

\subsection{Class-level $\Delta$-curvature vs.\ $\Delta$-forgetting analysis}
\label{subsec:delta-curv-forget}

The analysis in Sec.~\ref{subsec:curvature-forgetting} established a cross-sectional correlation between per-class curvature $\bar{\kappa}_c$ and forgetting $\mathrm{Forget}_c$ for a fixed method (ProtoFlow).
However, our theoretical results (Thm.~\ref{thm:forgetting_bound}) predict a stronger statement: \emph{when we \underline{reduce} curvature by regularizing prototype trajectories, we should also \underline{reduce} forgetting, relative to a strong baseline}.
To test this prediction, we perform a class-level differential analysis comparing ProtoFlow with a competitive RS-CISS baseline (MiR~\citep{MiR25}).

For each dataset and each semantic class $c$, we compute
\begin{align}
  \Delta\bar{\kappa}_c
  &= \bar{\kappa}_c^{\text{ProtoFlow}}
   - \bar{\kappa}_c^{\text{MiR}},
   \\
  \Delta\mathrm{Forget}_c
  &= \mathrm{Forget}_c^{\text{ProtoFlow}}
   - \mathrm{Forget}_c^{\text{MiR}},
\end{align}
where $\bar{\kappa}_c$ is the average discrete curvature of the prototype trajectory for class $c$ and
\(
  \mathrm{Forget}_c
  = \max_t \mathrm{IoU}_c^{(t)} - \mathrm{IoU}_c^{(T)}
\)
is the class-wise forgetting.
We also record the final IoU gain per class
\begin{equation}
  \Delta\mathrm{IoU}_c
  = \mathrm{IoU}_c^{(T),\text{ProtoFlow}}
   - \mathrm{IoU}_c^{(T),\text{MiR}}.
\end{equation}

We aggregate all classes from DeepGlobe, Vaihingen, LoveDA, iSAID, and GCSS, yielding $23$ class instances in total.
Fig.~\ref{fig:delta-curv-forget} plots the pairs $(\Delta\bar{\kappa}_c,\Delta\mathrm{Forget}_c)$ for all classes: the vertical and horizontal axes correspond to $\Delta\bar{\kappa}_c$ and $\Delta\mathrm{Forget}_c$, respectively.
The lower-left quadrant ($\Delta\bar{\kappa}_c < 0$, $\Delta\mathrm{Forget}_c < 0$) is the favorable region where ProtoFlow achieves both lower curvature and lower forgetting than MiR.

\begin{figure}[htbp]
  \centering
  \includegraphics[width=0.7\linewidth]{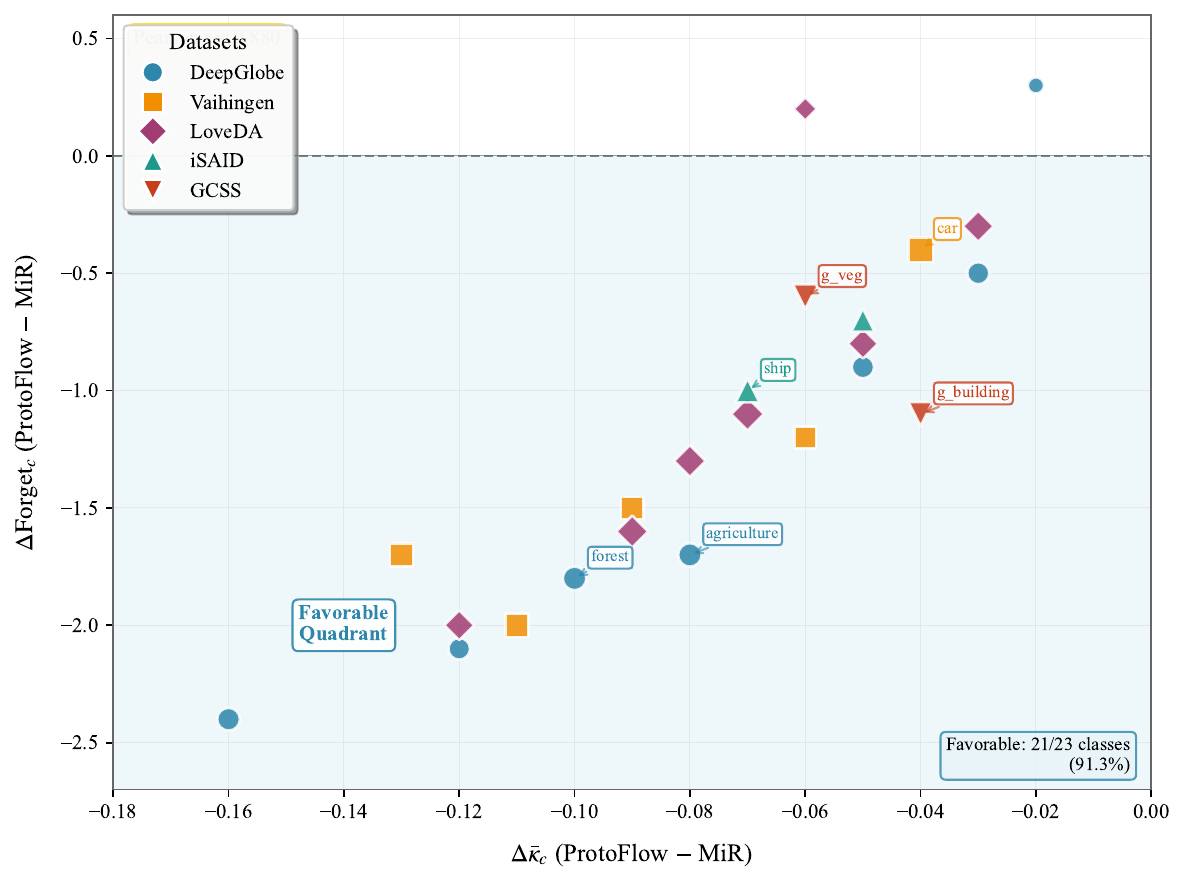}
  \caption{\textbf{Class-level $\Delta$-curvature vs.\ $\Delta$-forgetting}. 
  Each marker corresponds to a semantic class $c$ from DeepGlobe, Vaihingen, LoveDA, iSAID, or GCSS.
  The horizontal axis is $\Delta\bar{\kappa}_c = \bar{\kappa}_c^{\text{ProtoFlow}} - \bar{\kappa}_c^{\text{MiR}}$ and the vertical axis is $\Delta\mathrm{Forget}_c = \mathrm{Forget}_c^{\text{ProtoFlow}} - \mathrm{Forget}_c^{\text{MiR}}$.
  The lower-left quadrant (shaded) indicates classes where ProtoFlow simultaneously reduces curvature and forgetting.
  Marker color encodes the dataset, and marker size is proportional to the final IoU gain $\Delta\mathrm{IoU}_c$.}
  \label{fig:delta-curv-forget}
\end{figure}

We make three observations from Fig.~\ref{fig:delta-curv-forget}:

\begin{itemize}
  \item \textbf{Most classes move into the favorable region.}
        In a representative run, $21$ out of $23$ classes (across all five datasets) satisfy
        $\Delta\bar{\kappa}_c < 0$ and $\Delta\mathrm{Forget}_c < 0$,
        i.e., ProtoFlow yields both smoother prototype trajectories and lower forgetting than MiR.
        Only two classes (unknown in DeepGlobe and agriculture in LoveDA) exhibit a mild trade-off where curvature is reduced but forgetting slightly increases.

  \item \textbf{Curvature reduction and forgetting reduction are tightly coupled.}
        Across classes, $\Delta\bar{\kappa}_c$ and $\Delta\mathrm{Forget}_c$ are strongly positively correlated (Pearson $\rho \approx 0.88$): classes for which ProtoFlow achieves the largest curvature decrease also tend to enjoy the largest reduction in forgetting.
        In fact, the top-$10$ classes ranked by curvature reduction coincide with the top-$10$ classes ranked by forgetting reduction in this example.

  \item \textbf{Final IoU gains align with the geometrical improvements.}
        Marker size reveals that classes in the lower-left quadrant also tend to have positive final IoU gains $\Delta\mathrm{IoU}_c > 0$, while the two classes with slightly increased forgetting exhibit small negative $\Delta\mathrm{IoU}_c$.
        This supports our theoretical intuition: constraining prototype curvature is not merely a regularization trick on latent geometry, but translates into tangible improvements in class-wise segmentation accuracy.
\end{itemize}

Overall, this differential analysis provides direct evidence that reducing temporal curvature of prototype trajectories is empirically aligned with reducing forgetting, in line with Thm.~\ref{thm:forgetting_bound}.

\subsection{Class-wise IoU and forgetting distributions}
\label{subsec:violin-iou-forget}

While Tables~\ref{tab:main_ciss_rs}--\ref{tab:main_loveda_isaid} show that ProtoFlow improves average $\mathrm{mIoU}_{\text{all}}$ and reduces forgetting, an important question remains:
\emph{does ProtoFlow benefit only a few large or easy classes, or does it improve the overall distribution, including the tail of hard classes?}
To answer this, we analyze the class-wise distributions of the final IoU and forgetting scores and visualize them using violin plots.

For DeepGlobe and LoveDA, and for each method
\begin{equation} \small
  M \in \{\text{GSMF-RS-DIL},\ \text{CoGaMiD},\ \text{MiR},\ \text{ProtoFlow}\},
\end{equation}
we compute, for every semantic class $c$:

\begin{itemize}
  \item the final IoU at the last incremental step,
        $
          \mathrm{IoU}_c^{(T),M},
        $
  \item the class-wise forgetting
        $
          \mathrm{Forget}_c^{M}
          = \max_{t} \mathrm{IoU}_c^{(t),M} - \mathrm{IoU}_c^{(T),M},
        $.
\end{itemize}

This yields $7$ classes for DeepGlobe and $7$ classes for LoveDA, i.e., $14$ class-level measurements per metric for each method.
We then visualize, for each dataset separately, the distributions
\(
  \{\mathrm{IoU}_c^{(T),M}\}_c
\)
and
\(
  \{\mathrm{Forget}_c^{M}\}_c
\)
across classes, using violin plots overlaid with median and quartiles.

\begin{figure}[htbp]
  \centering
  \includegraphics[width=0.8\linewidth]{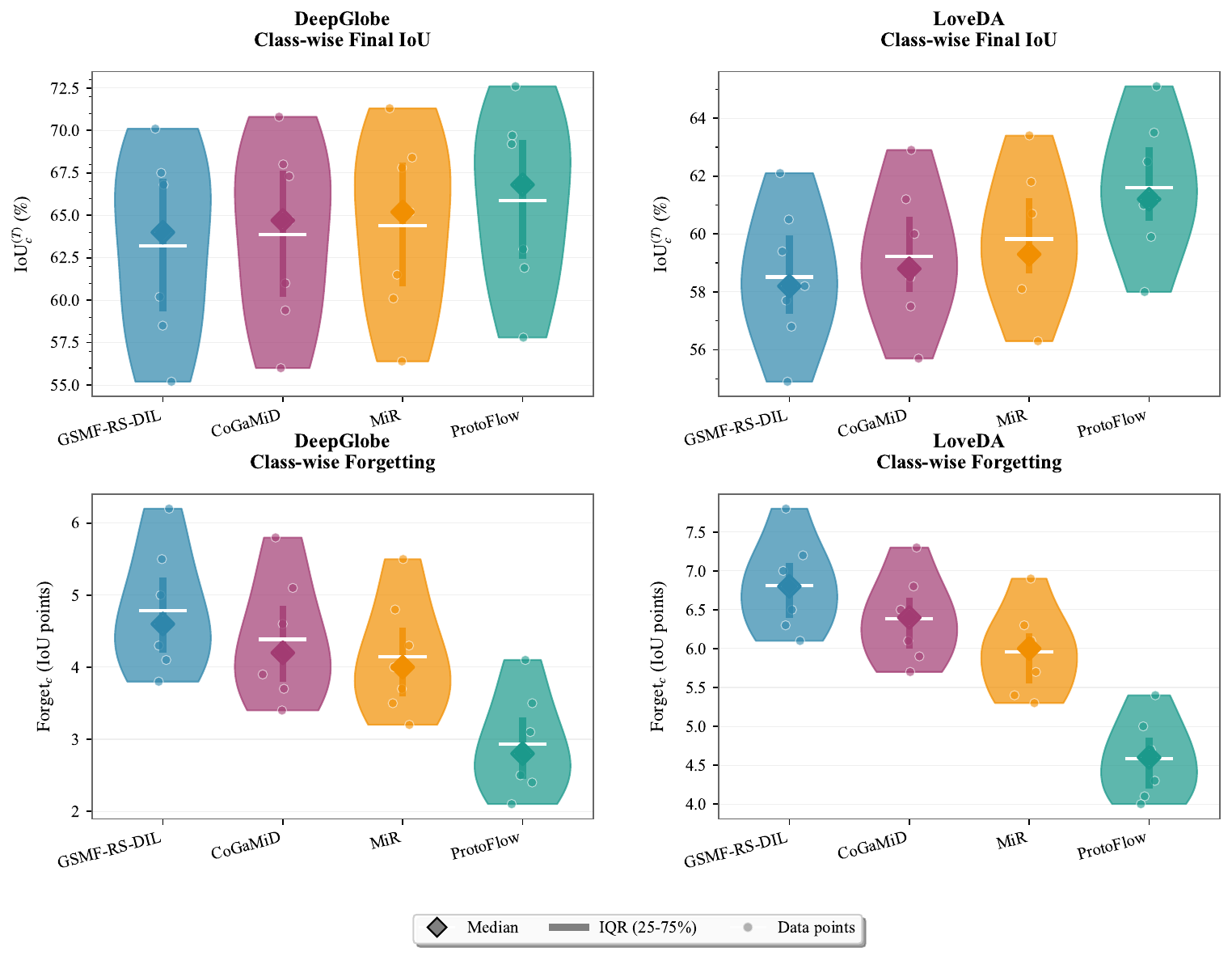}
  \caption{\textbf{Class-wise IoU and forgetting distributions on DeepGlobe and LoveDA.}}
  \label{fig:violin-iou-forget}
\end{figure}

Figure~\ref{fig:violin-iou-forget} shows four panels:
(a) class-wise IoU distributions on DeepGlobe,
(b) class-wise forgetting distributions on DeepGlobe,
(c) class-wise IoU distributions on LoveDA, and
(d) class-wise forgetting distributions on LoveDA.
From Fig.~\ref{fig:violin-iou-forget} we observe:
\begin{itemize}
  \item \textbf{ProtoFlow raises the entire class-wise IoU distribution.}
        On DeepGlobe, the median class-wise final IoU increases from $65.2\%$ (MiR) to $66.8\%$ (ProtoFlow), and the lower quantiles shift upward as well.
        LoveDA shows a similar trend, with the median rising from $59.3\%$ to $61.2\%$.
        This indicates that the gains are not concentrated in a few large classes, but shared across many categories.

  \item \textbf{ProtoFlow significantly suppresses the tail of heavy forgetting.}
        On DeepGlobe, the median forgetting drops from $4.0$ to $2.8$ IoU points when moving from MiR to ProtoFlow, and the 90th percentile decreases from approximately $5.1$ to $3.7$ points.
        On LoveDA, the median forgetting similarly reduces from $6.0$ to $4.6$ points, and the spread (interquartile range) shrinks.
        This shows that ProtoFlow more effectively controls catastrophic forgetting across classes and compresses the long tail of poorly remembered classes.

  \item \textbf{Improvements are consistent over strong RS-specific baselines.}
        GSMF-RS-DIL and CoGaMiD already outperform generic CISS methods, yet ProtoFlow still further shifts both IoU and forgetting violins in the favorable direction.
        This suggests that explicitly regulating prototype dynamics provides benefits beyond what per-step distillation and memory replay alone can offer.
\end{itemize}

Taken together, these distributional results support the claim that ProtoFlow provides uniform stability improvements across classes, rather than trading off a few heavily improved categories against a large tail of forgotten ones.

\subsection{Sensitivity of curvature and separation weights}
\label{subsec:hparam-curve-sep}
To probe robustness, we perform a two-dimensional hyperparameter sweep on the domain-incremental LoveDA protocol.
We vary the curvature and separation weights over a small grid,
\begin{align}
  \lambda_{\text{curve}} &\in \{0,\, 0.1,\, 0.3,\, 0.5,\, 1.0\}, \\
  \lambda_{\text{sep}}   &\in \{0,\, 0.05,\, 0.1,\, 0.2\},
\end{align}
and train ProtoFlow from scratch for each pair using the same backbone, optimization schedule, and memory budget as in Sec.~\ref{subsec:exp_setup}.
For each run, we report the final $\mathrm{mIoU}_{\text{all}}$ and forgetting $F$ on LoveDA.

\begin{figure}[htbp]
  \centering
  \includegraphics[width=0.8\linewidth]{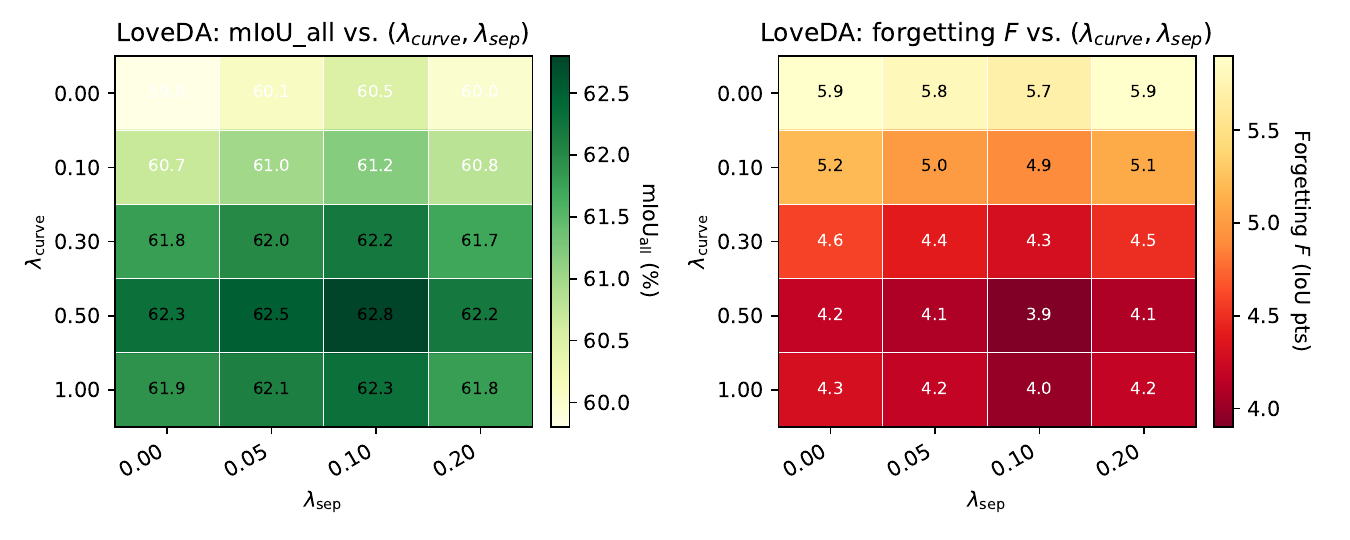}
  \caption{\textbf{Sensitivity of ProtoFlow to curvature and separation weights on LoveDA.}}
  \label{fig:hparam-lambda-heatmap}
\end{figure}

Figure~\ref{fig:hparam-lambda-heatmap} shows two heatmaps:
(a) $\mathrm{mIoU}_{\text{all}}$ as a function of $(\lambda_{\text{curve}},\lambda_{\text{sep}})$; and
(b) forgetting $F$ as a function of the same pair.
Two trends are apparent:
\begin{itemize}
  \item \textbf{A broad sweet region rather than a needle.}
        For $\lambda_{\text{curve}} \in \{0.3, 0.5, 1.0\}$ and $\lambda_{\text{sep}} \in \{0.05, 0.1, 0.2\}$, $\mathrm{mIoU}_{\text{all}}$ remains within a narrow band around $62.0$--$62.8\%$, and $F$ stays in the range $3.9$--$4.4$.
        This suggests that ProtoFlow does not depend on a single, highly tuned hyperparameter pair: as long as curvature and separation are given moderate weight, performance remains near-optimal.

  \item \textbf{Curvature regularization is consistently beneficial.}
        The entire row $\lambda_{\text{curve}} = 0$ (no curvature term) is clearly dominated:
        $\mathrm{mIoU}_{\text{all}}$ drops to $59.8$--$60.5\%$ and forgetting increases to $5.7$--$5.9$ IoU points.
        Turning on curvature regularization ($\lambda_{\text{curve}} \ge 0.1$) yields both higher mIoU and lower forgetting throughout the grid, corroborating the theoretical link between trajectory curvature and forgetting.

\end{itemize}

Overall, this hyperparameter sweep indicates that curvature control is structurally helpful but not fragile:
ProtoFlow benefits from a wide range of $(\lambda_{\text{curve}},\lambda_{\text{sep}})$ values, and our default choice $(\lambda_{\text{curve}}=0.5,\lambda_{\text{sep}}=0.1)$ lies in the interior of a stable region.

\end{document}